\pdfoutput=1

\documentclass[10pt,twocolumn,letterpaper]{article}

\usepackage[pagenumbers]{cvpr} % To force page numbers, e.g. for an arXiv version

\usepackage{amsthm} 

\newtheorem{lemma}{Lemma}

\usepackage{multirow}
\usepackage{colortbl}
\usepackage{xcolor}
\newlength\savewidth\newcommand\shline{\noalign{\global\savewidth\arrayrulewidth
  \global\arrayrulewidth 1pt}\hline\noalign{\global\arrayrulewidth\savewidth}}
\newcommand{\tablestyle}[2]{\setlength{\tabcolsep}{#1}\renewcommand{\arraystretch}{#2}\centering\footnotesize}
\renewcommand{\paragraph}[1]{\vspace{1.25mm}\noindent\textbf{#1}}

\definecolor{cvprblue}{rgb}{0.21,0.49,0.74}
\usepackage[pagebackref,breaklinks,colorlinks,allcolors=cvprblue]{hyperref}

\def\paperID{4120} % *** Enter the Paper ID here
\def\confName{CVPR}
\def\confYear{2026}

\title{RNN as Linear Transformer: A Closer Investigation into Representational Potentials of Visual Mamba Models}

\author{Timing Yang$^1$ \quad Guoyizhe Wei$^1$ \quad Alan Yuille$^1$ \quad Feng Wang$^{1*}$\\
\\
$^1$Johns Hopkins University
}

\begin{document}

\twocolumn[{%
\renewcommand\twocolumn[1][]{#1}%
\maketitle
\vspace{-1cm}
\begin{center}
    \captionsetup{type=figure, width=1\textwidth}
    \includegraphics[width=1\textwidth]{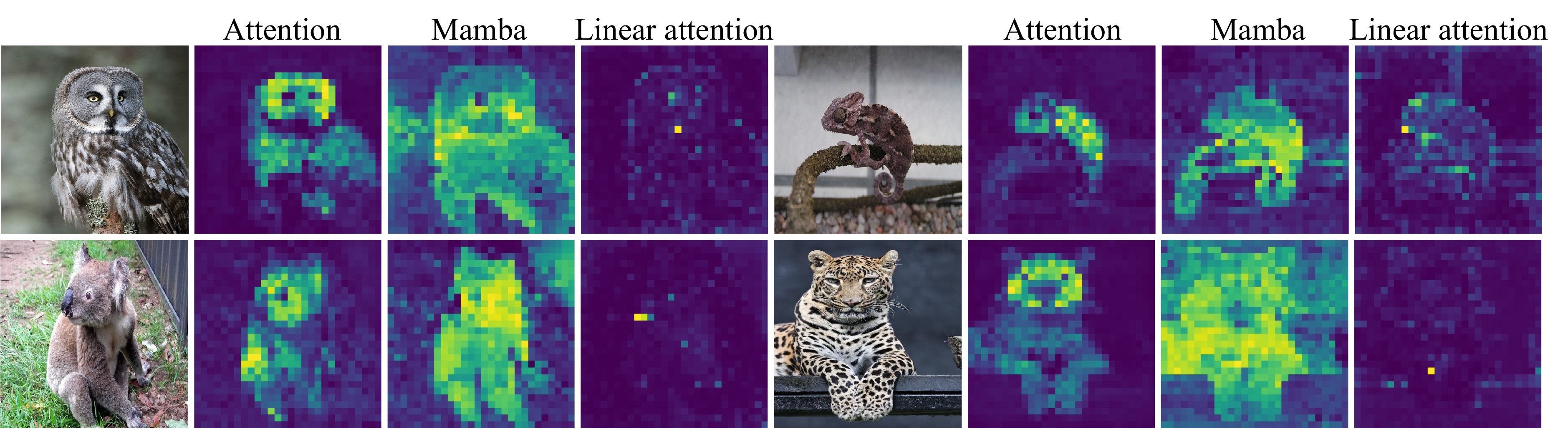}
    
\captionof{figure}{\textbf{Activation maps of Self-Attention~\cite{attention}, Mamba~\cite{vim}, and Linear Attention~\cite{linearattn}.} As shown, Self-Attention typically produces high-quality activations with clear foreground-background distinction, while Mamba shows similar patterns but noisier background activations. Linear Attention, in contrast, often struggles to clearly focus on the informative parts of the images.}
    
    \label{fig:Vis-attn-linearattn-mamba}
\end{center}%
}]

\maketitle
\begin{abstract}

\def\thefootnote{*}\footnotetext{Corresponding author. Email: \url{wangf3014@gmail.com}. Code is available at \url{https://github.com/yangtiming/Dino-Mamba}.}

\vspace{-0.2cm}
Mamba has recently garnered attention as an effective backbone for vision tasks. However, its underlying mechanism in visual domains remains poorly understood. In this work, we systematically investigate Mamba’s representational properties and make three primary contributions. First, we theoretically analyze Mamba’s relationship to Softmax and Linear Attention, confirming that it can be viewed as a low-rank approximation of Softmax Attention and thereby bridging the representational gap between Softmax and Linear forms. Second, we introduce a novel binary segmentation metric for activation map evaluation, extending qualitative assessments to a quantitative measure that demonstrates Mamba’s capacity to model long-range dependencies. Third, by leveraging DINO for self-supervised pretraining, we obtain clearer activation maps than those produced by standard supervised approaches, highlighting Mamba’s potential for interpretability. Notably, our model also achieves a 78.5\% linear probing accuracy on ImageNet, underscoring its strong performance. We hope this work can provide valuable insights for future investigations of Mamba-based vision architectures.

\end{abstract}  
\section{Introduction}
\label{sec:intro}

Transformers have made impressive strides from natural language processing (NLP) to vision tasks, thanks to their ability to model global dependencies. However, their quadratic time and memory complexity poses significant challenges in scaling to high-resolution images. This limitation spurs interest in alternative sequence models such as State Space Models (SSMs) \cite{ssm1,gu2021combining,hippo,mehta2022long}, which capture long-range dependencies with linear complexity. Among these alternatives, the Mamba \cite{mamba} model emerges as a particularly promising approach by introducing selective scanning techniques and hardware-aware optimizations that enable efficient training and inference. Its next iteration \cite{mamba2} further incorporates State Space Duality (SSD), allowing larger state dimensions and faster training times. Following the trajectory of transformers \cite{attention}’ expansion from NLP to vision, Mamba is successfully applied to a range of vision tasks, including neural architecture design \cite{vim,vmamba,mamba-r,wang2024causal}, semantic segmentation and object detection \cite{yolomamba,rs3mamba}, and medical image analysis \cite{U-mamba,vim4path}.

While transformers have been extensively studied and their global modeling capabilities are well established, the ability of the relatively novel Mamba \cite{mamba,mamba2} architecture to capture long-range dependencies with a linear mechanism remains in question. To address this concern, we theoretically analyze Mamba’s representational capacity in comparison to self-attention and linear attention. First, we unify the formulations of Mamba \cite{mamba2}, self-attention \cite{attention}, and linear attention \cite{linearattn} into a common mathematical framework for direct comparison of their expressiveness via matrix rank analysis. In self-attention \cite{attention}, applying a softmax to the $QK^T$ attention matrix makes its rank effectively full (up to the sequence length limit). In contrast, linear attention \cite{linearattn} uses a kernel-based mechanism in place of softmax, restricting the attention matrix’s rank to the dimensionality of the query/key projections (the head dimension). This reduction in rank yields lower computational complexity at the cost of representational power. Mamba \cite{mamba2}, though structurally similar to linear attention, introduces a learnable, data-dependent causal mask in place of the fixed mask, thereby significantly enhancing its expressiveness. This learned transformation preserves a high rank in Mamba’s state-space matrices, enabling it to capture richer long-range dependencies while maintaining linear efficiency. Consequently, Mamba achieves a balance between representational capacity and efficiency, effectively positioning itself between self-attention and linear attention in terms of capability. 

Next, to validate these theoretical findings, we analyze Mamba’s feature representations in practice through visualization. Visualizing feature maps provides an intuitive means to understand what each model encodes and how it processes information. In particular, we leverage DINO \cite{Dino}—a self-supervised vision transformer known for producing clear and interpretable feature maps — to compare the representations learned by Mamba, self-attention, and linear attention. As illustrated in Figure \ref{fig:Vis-attn-linearattn-mamba}, self-attention yields the most distinct separation between foreground and background regions; Mamba exhibits similar global attention patterns albeit with noisier background activations; and linear attention often struggles to clearly delineate object boundaries. These empirical observations corroborate our rank analysis, indicating that Mamba’s higher-rank representations do capture broad context, even if some background noise remains in its activations.

While such visual comparisons are informative, evaluating feature quality requires quantitative rigor. However, existing evaluation methods \cite{Visiontransformersneedregisters,Dino,dinov2} largely rely on subjective inspection of feature maps, lacking objective metrics for comparison. To address this gap, we propose the Binary-AUC Metric, a new AUC-based evaluation that quantitatively measures feature map discriminability.  This metric computes the area under the ROC curve (AUC) by comparing model-generated feature maps against ground-truth segmentation masks, providing an objective measure of representational quality. Using this metric, we find that self-attention produces the most discriminative feature maps, followed by Mamba, with linear attention performing the poorest. Notably, feature maps from self-supervised training (DINO) are clearer than those from supervised training, and their quality further improves with larger model size and depth. 

Finally, we integrate Mamba into the DINO self-supervised framework to assess its high-level performance on vision tasks. In a linear probing evaluation, Mamba achieves classification accuracy on par with transformer-based models. More importantly, on downstream tasks that benefit from long-sequence or high-resolution data, Mamba significantly outperforms Linear Attention across detection and segmentation tasks, demonstrating its superior ability to model long-range dependencies while maintaining linear complexity. These findings underscore Mamba’s potential as a robust, efficient, and effective alternative to transformers in self-supervised vision applications. 

Our contributions can be summarized as follows:

\begin{itemize}  
    \item We theoretically analyze Mamba via its token-mixing properties and show that it can be viewed as a low-rank approximation of Softmax Attention. Although this indicates a weaker representational capacity compared to Softmax Attention, Mamba offers greater potential and scalability than Linear Attention, making it an attractive middle ground for efficient vision architectures.
    
    \item We introduce a binary segmentation metric designed to quantitatively assess the interpretability of activation maps. By applying this metric, we are able to rigorously evaluate how Mamba captures long-range dependencies, moving beyond purely qualitative visualization methods to a more objective, measurement-driven analysis.
    
    \item We leverage DINO for self-supervised pretraining, yielding  clearer activation maps than those obtained via standard supervised training, by which our Vision Mamba model achieves a competitive 78.5\% linear probing accuracy on ImageNet, underscoring both its representational potential and practical efficacy in vision tasks.
    
\end{itemize}

\section{Related Work}
\label{sec:Related Work}

\paragraph{State space models.} State Space Models (SSMs), originally used for continuous inputs in control systems~\citep{ssm-control}, have gained traction in deep learning through recent discretization advances~\citep{ssm-dis1,hippo,ssm-dis2,ssm-dis3}. SSMs encompass recurrent models with latent states, from traditional Hidden Markov Models~\citep{hmm} and RNNs to modern variants: Linear Attention~\citep{linearattn} achieves linear complexity via softmax-free mechanisms, while RWKV~\citep{rwkv} combines Transformer training efficiency with RNN inference. Mamba~\cite{mamba} introduces data-dependent selective SSMs with hardware optimization, further enhanced by Mamba-2~\cite{mamba2} through State Space Duality (SSD) for improved scalability.  Mamba's advent has prompted substantial research~\cite{han2024demystify,chou2024metala,ali2025hidden,han2024bridging,yang2024parallelizing} exploring relationships between attention mechanisms (softmax and linear) and state space models(\eg mamba). However, most of these studies either analyze Mamba through the lens of gating mechanisms or reformulate it as a variant of attention mechanisms. None of these works analyze Mamba from its own perspective; instead, they merely bring Mamba closer to the attention framework. For instance, MILA~\cite{han2024demystify} links Mamba to linear attention, attributing its success to the forget gate, while RALA~\cite{fan2024rala} addresses low-rank limitations in linear attention through rank augmentation. Unlike prior work that interprets Mamba through attention mechanisms, we analyze Mamba from its own perspective to assess its representational capacity. This intrinsic analysis provides the foundation for our feature map visualizations and experimental validation.

\paragraph{Mamba in visual application.} Mamba is extended to the field of computer vision~\cite{mambavision, mamba-r, vim, li2024videomamba, yolomamba, log-vmamba, patchscaling,mvar,zhao2025ud,arm,adventurer}, analogous to the successful adaptation of Transformers from NLP to vision-related tasks. For instance, Vim \cite{vim} is the first models in vision tasks to adopt a bidirectional forward and backward scanning approach to resolve the issue where Mamba's inherent directionality limits learning from preceding tokens. Vmamba \cite{vmamba}, viewed as a model combining convolution with Mamba blocks, adopts a hybrid structure by stacking Visual State-Space blocks and the 2D Selective Scan module to capture contextual information from multiple directions and dimensions. Mamba-Reg \cite{mamba-r} similarly inserts multiple register tokens \cite{Visiontransformersneedregisters} to remove outliers and enhance performance. While previous Mamba architectures focus on model design, we explore Visual Mamba's underlying mechanisms through feature map visualization.

\paragraph{Self-supervised learning} for visual representations becomes an influential approach to utilizing unlabeled data for learning robust and transferable features, including context based \cite{rotation,Jigsaw,Colorization}, contrastive learning \cite{SimCLR,SwAV,BYOL,Dino,dinov2}, and masked image modeling \cite{mae,simmim,beit}, etc. Compared to supervised learning, self-supervised learning with DINO \cite{Dino} demonstrates the ability to produce exceptionally clean feature maps. Both \cite{Visiontransformersneedregisters,mamba-r} explore the impact of introducing register tokens in Vision, showing that these tokens effectively denoise feature maps. However, both works lack quantitative feature map evaluation. While existing CNN studies~\cite{FQuality1,FQuality2,FQuality3} use regression and entropy metrics, our CLS-token-based Binary-AUC approach is the first ImageNet-scale assessment, providing task-relevant insights into Mamba's mechanisms.

\section{Theoretical Analysis }
\label{sec:TheoreticalAnalysis}

\subsection{Basic Formulations}

\paragraph{Softmax self-attention} captures content-based interactions within a sequence. Given an input \( \mathcal{X} \in \mathbb{R}^{L \times d_{\text{model}}} \), the query, key, and value matrices are computed as
\begin{equation}
\mathbf{Q} = \mathcal{X} W^Q,\quad \mathbf{K} = \mathcal{X} W^K,\quad \mathbf{V} = \mathcal{X} W^V,
\end{equation}
where \( \mathbf{Q}, \mathbf{K} \in \mathbb{R}^{L \times D_{QK}} \) and \( \mathbf{V} \in \mathbb{R}^{L \times J} \), and the self-attention output is given by
\begin{equation}
\mathbf{Y} = \mathrm{softmax}\left(\mathbf{Q}\mathbf{K}^\top/\sqrt{D_{QK}}\right)\mathbf{V}.
\label{eq:self-attn}
\end{equation}
The softmax function normalizes each row of $\mathbf{Z}$ as $\mathrm{softmax}(\mathbf{Z})_{i,j} = \exp(Z_{i,j}) / \sum_{k=1}^{L} \exp(Z_{i,k})$, ensuring $Y_i$ is a weighted sum of $\mathbf{V}$ based on query-key similarities, with complexity $O(L^2)$.

\paragraph{Linear attention} reduces the \(O(L^2)\) complexity of self-attention, by replacing the \(\mathrm{softmax}(\mathbf{QK}^\top)\) term with a kernel-based factorization \(\psi(\mathbf{Q}) \psi(\mathbf{K})^\top\), leading to
\begin{equation}  
\mathbf{Y} = (\mathbf{\psi(Q)\psi(K)}^\top) \mathbf{V}, 
\end{equation}
where \(\psi(\cdot)\) projects queries and keys into a feature space, approximating softmax attention while reducing complexity to \(O(L)\). 
For autoregressive settings, causality is enforced via a causal mask \(\mathbf {L_{Attn}} \in \mathbb{R}^{L \times L}\), leading to
\begin{equation}
\label{causallinearattn}
\mathbf{Y} = (\mathbf {L_{Attn}} \circ \mathbf{\psi(Q)\psi(K)}^\top) \mathbf{V},
\end{equation}
where \(\circ\) denotes element-wise multiplication, and \(\mathbf{L_{Attn}}\) is a lower triangular matrix with ones in its lower triangular part and zeros elsewhere, ensuring that each token only attends to previous tokens.

\paragraph{Mamba.} The earlier version of Mamba \cite{mamba} utilizes the structure of SSMs, which maps a one-dimensional input \( x \in \mathbb{R}^{\mathrm{L}} \) to an output \( y \in \mathbb{R}^{\mathrm{L}} \) via a hidden state \( h_t \in \mathbb{R}^N \), where $\mathrm{L}$ is sequence length and $N$ is state dimension. The process can be expressed as:
\begin{equation} \label{discrete-ssm}
\left\{
\begin{array}{l}
h_t = \boldsymbol{\mathcal{A}}_t h_{t-1} + \mathbf{B}_t x_t, \\
y_t = \mathbf{C}^{\top}_t h_t.
\end{array}
\right.
\end{equation}
Since the parameters of an SSM model evolve over time, we can consider \(\boldsymbol{\mathcal{A}} \in \mathbb{R}^{L \times N \times N}\), \(\mathbf{B} \in \mathbb{R}^{L \times N}\) and \(\mathbf{C} \in \mathbb{R}^{L \times N}\).

The SSM can be implemented in a convolution form as follows. Starting with the initial state \( h_0 = B_0 x_0 \), the recursive state-space update equation can be expanded as
\begin{equation}
h_t = \sum_{i=0}^{t} \mathcal{A}_t \mathcal{A}_{t-1} \dots \mathcal{A}_{i+1} B_i x_i
\end{equation}
and the output at time step \( t \) is  
\begin{equation}
\label{conv-ssmsformat}
y_t = \sum_{i=0}^{t} C_t^\top \mathcal{A}^{\times}_{t:i} B_i x_i.
\end{equation}
This equation can be written in a genral form of
\begin{equation}
\label{SSMs-simple}
\mathbf{Y} = \text{SSM}(\boldsymbol{\mathcal{A}}, \mathbf{B}, \mathbf{C})(\mathbf{X})
\end{equation}
where \(\mathbf{Y}, \mathbf{X} \in \mathbb{R}^{L \times J}\), and \(\mathbf{A} \in \mathbb{R}^{L}\) following Mamba-2's formulation~\cite{mamba2}. Equation (\ref{SSMs-simple}) can be then written in a matrix transformation
\begin{equation}
\label{mamba-f}
\mathbf{Y} = \mathbf{M}\mathbf{X}, 
\mathbf{M}_{ij} =
\left\{
\begin{array}{ll}
\mathbf{C}_i^\top \mathbf{A}_i \cdots \mathbf{A}_{j+1} \mathbf{B}_j, & i \geq j, \\
0, & i < j,
\end{array}
\right.
\end{equation}
where \(\mathbf{A}_i\) is a scalar, and \(\mathbf{M}\) can be represented by
\begin{equation}
\label{eq:L-causal_m}
\mathbf{M}_{ij} = \mathbf {L_{M}}_{ij} \circ (\mathbf{C}_i^\top \mathbf{B}_j), \:
\mathbf {L_{M}}_{ij} = \mathbf{A}_i \cdots \mathbf{A}_{j+1},
\end{equation}
where \( \mathbf {L_{M}} \in \mathbb{R}^{L \times L} \) acts as a learnable causal mask controlling the impact of the state transition matrix.

By this point, we have found that softmax self-attention, linear attention, and Mamba can all be formulated using \textbf{\textit{the same unified expression}} of
\begin{equation}
    \mathbf{Y} = \mathbf{M}\mathbf{X},
\end{equation}
where
\begin{equation}
\label{Unified-f}
\mathbf{M} =
\left\{
\begin{array}{l@{\quad}l}
\mathbf{L_M} \circ (\mathbf{C}^\top \mathbf{B}), & \text{Mamba} \\
\mathbf{L_{Attn}} \circ \text{softmax}(\mathbf{Q} \mathbf{K}^\top), & \text{Self-Attn} \\
\mathbf{L_{Attn}} \circ (\psi(\mathbf{Q}) \psi(\mathbf{K})^\top), & \text{Lin-Attn}
\end{array}
\right.
\end{equation}
represents how the three distinct token mixers process there input $\mathbf{X}$ into output $\mathbf{Y}$ (shown in Figure~\ref{fig:Lowertriangularmatrix}. Note that here we keep the causal attention mask \(\mathbf{L_{Attn}}\) for softmax self-attention and linear attention to guarantee their $\mathbf{M}$ being a lower-triangle matrix. In the following analysis, we will divide $\mathbf{M}$ into sub-matrices and the causal mask does not affect the rank of off-diagonal blocks.

\begin{figure}[h!]
    \centering
    \vspace{+0.2cm}
    \includegraphics[width=1.0\linewidth]{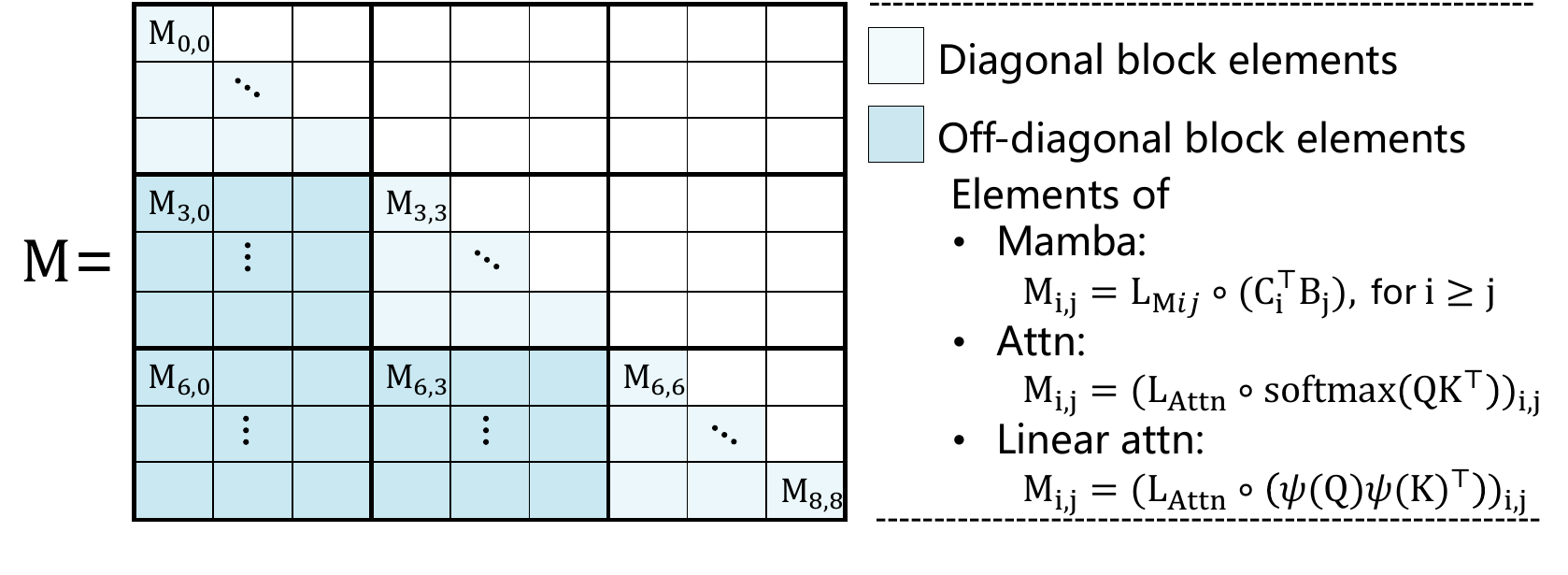}
    \vspace{-0.8cm}
    \caption{\textbf{Unified formulation} of the lower triangular matrix \( \textbf{M} \) with diagonal and off-diagonal block elements. This example is for a sequence length of nine and a chunk size of three.}
    \label{fig:Lowertriangularmatrix}
\end{figure}

\subsection{Hadamard Product Rank Bounds}

The following lemma establishes rank bounds for the Hadamard (element-wise) product of matrices, which is crucial for analyzing the structure of $\mathbf{M}$ in our unified formulation. This bound is a standard result in matrix analysis. \textbf{Note}: \textit{The following lemma uses independent abstract notations for clarity.}
\begin{lemma}[Hadamard Product Rank Bound]
\label{lem:hadamard_rank}
For two matrices $\mathbf{A}$ and $\mathbf{B}$ of the same dimensions, the Hadamard product $\mathbf{A} \circ \mathbf{B}$ (defined element-wise as $(\mathbf{A} \circ \mathbf{B})_{i,j} = \mathbf{A}_{i,j}\mathbf{B}_{i,j}$) satisfies:
\[
\operatorname{rank}(\mathbf{A} \circ \mathbf{B}) \le \operatorname{rank}(\mathbf{A}) \cdot \operatorname{rank}(\mathbf{B}),
\]
\end{lemma}
\begin{proof}
Let $r_{\mathbf{A}} = \operatorname{rank}(\mathbf{A})$ and $r_{\mathbf{B}} = \operatorname{rank}(\mathbf{B})$. Using the Singular Value Decomposition (SVD), we can write:
\[
\mathbf{A} = \sum_{i=1}^{r_{\mathbf{A}}} \mathbf{p}_i \mathbf{q}_i^{\top}, \qquad \mathbf{B} = \sum_{j=1}^{r_{\mathbf{B}}} \mathbf{s}_j \mathbf{t}_j^{\top},
\]
where $\mathbf{p}_i, \mathbf{q}_i$ (for $i = 1, \ldots, r_{\mathbf{A}}$) and $\mathbf{s}_j, \mathbf{t}_j$ (for $j = 1, \ldots, r_{\mathbf{B}}$) are appropriately sized vectors. The Hadamard product can be expressed as:
\[
\mathbf{A} \circ \mathbf{B} = \sum_{i=1}^{r_{\mathbf{A}}} \sum_{j=1}^{r_{\mathbf{B}}} (\mathbf{p}_i \circ \mathbf{s}_j)(\mathbf{q}_i \circ \mathbf{t}_j)^{\top},
\]
where $\mathbf{p}_i \circ \mathbf{s}_j$ denotes the element-wise product of vectors. This shows that $\mathbf{A} \circ \mathbf{B}$ is a sum of $r_{\mathbf{A}} \cdot r_{\mathbf{B}}$ rank-one matrices. Therefore:
\begin{align}
\operatorname{rank}(\mathbf{A} \circ \mathbf{B}) \le r_{\mathbf{A}} \cdot r_{\mathbf{B}} = \operatorname{rank}(\mathbf{A}) \cdot \operatorname{rank}(\mathbf{B}).
\end{align}
\end{proof}

\subsection{Rank Analysis of Attention and Mamba}

We follow Mamba-2~\cite{mamba2} to formulate the SSM as a semi-separable matrix using block matrix decomposition, which divides $\mathbf{M}$ into diagonal blocks for within-chunk computations and off-diagonal blocks for cross-time dependencies. To assess representational capacity, we partition $\mathbf{M}$ into a $L/C \times L/C$ grid of $C \times C$ sub-matrices and analyze their rank for a single head. We denote the rank of diagonal and off-diagonal sub-matrices as $R^{\text{diag}}$, $R^{\text{off}}$, respectively, where $\mathbf{L^{\text{off}}_{Attn}}$, $\mathbf{L^{\text{off}}_M}  \in \mathbb{R}^{C \times C}$ represent the off-diagonal portions of causal masks. For diagonal sub-matrices, the lower triangular structure ensures full rank for all three mechanisms, leading to guaranteed $R^{\text{diag}} = C$. For off-diagonal sub-matrices, \textit{Softmax Self-Attention} achieves full rank due to the nonlinear effect of softmax, which transforms $\mathbf{L^{\text{off}}_{Attn}}$ from a fixed to a learnable mask $R^{\text{off}}_{\text{Self-Attn}} = C$. According to Lemma~\ref{lem:hadamard_rank}, the rank of \textit{Linear Attention} and \textit{Mamba} subject to:
\begin{equation}
\begin{cases}
R^{\text{off}}_{\text{LinAttn}}(\mathbf{M}) \le \operatorname{rank}(\mathbf{L^{\text{off}}_{Attn}}) \cdot \operatorname{rank}(\psi(\mathbf{Q})\psi(\mathbf{K})^\top), \\
R^{\text{off}}_{\text{Mamba}}(\mathbf{M}) \le \operatorname{rank}(\mathbf{L^{\text{off}}_M}) \cdot \operatorname{rank}(\mathbf{C}^\top\mathbf{B}).
\end{cases}
\label{eq:rank_ineq}
\end{equation}

Analyzing each component: the causal mask $\mathbf{L^{\text{off}}_{Attn}}$ consists of all ones and is fixed, while $\mathbf{L^{\text{off}}_M}$ (Equation~\eqref{eq:L-causal_m}) is learnable, introducing nonlinearity within each submatrix.
In addition, given in Lemma \ref{lem:matrix_product_rank} (Appendix), in Linear Attention, the rank of $\psi(\mathbf{Q}) \psi(\mathbf{K})^\top$ is constrained by $D_{QK}$, whereas in Mamba, the rank of $\mathbf{C}^\top \mathbf{B}$ is bounded by $N$, where $N$ denotes the state dimension. Therefore, we have:
\begin{equation}
\begin{cases}
\text{rank}(\mathbf{L^{\text{off}}_{Attn}}) = 1, \quad \text{rank}(\psi(\mathbf{Q})\psi(\mathbf{K})^\top) \le D_{QK}, \\
\text{rank}(\mathbf{L^{\text{off}}_M}) \ge 1, \quad \text{rank}(\mathbf{C}^\top\mathbf{B}) \le N
\end{cases}
\label{eq:rank_comp}
\end{equation}

Combining these results from~\eqref{eq:rank_ineq} and~\eqref{eq:rank_comp}, we obtain:
\begin{equation}
R^{\text{off}}_{\text{LinAttn}} \le D_{QK}, \quad 
R^{\text{off}}_{\text{Mamba}} \le \operatorname{rank}(\mathbf{L^{\text{off}}_M}) \cdot N
\label{eq:rank_final}
\end{equation}
Through rank analysis, all three methods share $R^{\text{diag}} = C$, but differ in off-diagonal rank upper bounds:
\begin{equation}
R^{\text{off}}: \quad \underbrace{C}_{\text{Self-Attn}} \,>\, \underbrace{\text{rank}(\mathbf{L^{\text{off}}_M}) \cdot N}_{\text{Mamba upper bound}} \,>\, \underbrace{D_{QK}}_{\text{Lin-Attn upper bound}}
\end{equation}
For typical base models with $C=256$, $N=D_{QK}=64$, it establishes a \textbf{\textit{performance hierarchy of Softmax-self-attention $>$ Mamba $>$ Linear-attention}}. This configuration is the de facto standard across modern vision architectures~\cite{Dino,dinov2,vim,vmamba,linearattn,attention,deit}, ensuring broad applicability of our analysis. Critically, Mamba's $O(LNJ)$ complexity allows $N$ to scale efficiently, whereas Linear Attention's $O(LD_{QK}^2)$ quadratic cost restricts $D_{QK}$ growth. This scalability difference ensures that in typical deployments, $\text{rank}(\mathbf{L^{\text{off}}_M}) \cdot N > D_{QK}$, maintaining the rank hierarchy across practical model configurations. The detailed proof is provided in the Appendix Mathematical Foundations for Rank Analysis.

\section{Evaluation Metrics}
\label{sec:Method}

\subsection{Binary-AUC Metric for Feature Maps} 

Following our rank-based analysis, we extend this perspective to feature maps in vision tasks. The similarity matrix between the \texttt{[CLS]} token and image tokens can be viewed as a linear transformation, where the effective rank reflects the model’s ability to distinguish image regions. A higher rank suggests richer feature representation, while a lower rank indicates redundancy or information loss.
To further validate this, we analyze the structure and spectral properties of the feature map.

\begin{figure}[h!]
    \centering
    \includegraphics[width=1.0\linewidth]{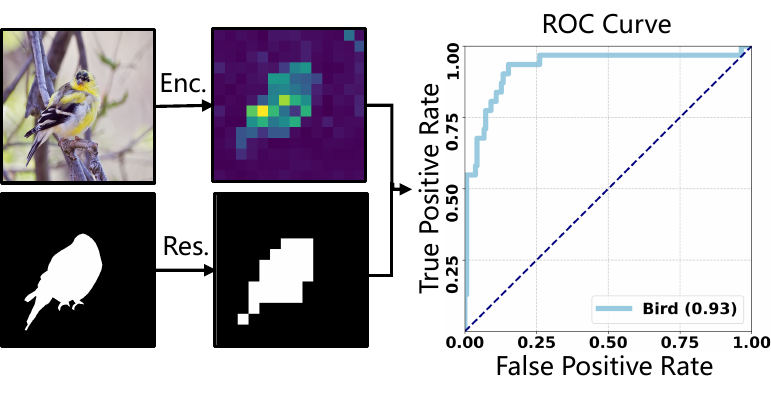}
   
    \caption{\textbf{Visualization metric}: The process involves encoding (Enc.) an image, resizing (Res.) the segmentation mask to align with the feature map, and calculating the AUC scores.}
   
    \label{fig:VisualizationMetricforFeatureMap}
\end{figure}

While feature map visualization is crucial for understanding model representations, existing studies~\cite{dinov2,Dino,Visiontransformersneedregisters,mamba-r} rely solely on visual inspection. We introduce Binary-AUC Metric, an AUC-based quantitative evaluation using segmentation labels (Figure~\ref{fig:VisualizationMetricforFeatureMap}). We compute AUC between feature map activations and ground truth segmentation, merging multiple labels into a foreground/background mask. For patch size $N_P$ and image size $H \times W$, the feature map size is $(H/N_P, W/N_P)$. After resizing the segmentation mask, we define binary masks $\text{Mask}_{\text{feature}}$ and $\text{Mask}_{\text{label}}$, generating activated masks for each threshold $t \in [0, 1]$:
\begin{equation}
\text{Mask}_{\text{feature}}^t(i, j) = 
\left\{
\begin{array}{ll}
1, & \text{if } \text{feature}(i, j) \geq t \\
0, & \text{otherwise}
\end{array}
\right.
\end{equation}
At each threshold $t$, we calculate the rates:
\begin{equation}
R(t, S) = \frac{|\text{Mask}_{\text{feature}}^t \cap S|}{|S|}
\end{equation}
where $\text{TPR}(t) = R(t, \text{Mask}_{\text{label}})$ measures the true positive rate and $\text{FPR}(t) = R(t, \overline{\text{Mask}_{\text{label}}})$ measures the false positive rate.
By varying $t$, we obtain the ROC curve. The area under the curve (AUC) is computed as:
\begin{equation}
\text{AUC} = \sum_{i} (\text{FPR}_{i+1} - \text{FPR}_{i}) \cdot \frac{\text{TPR}_{i+1} + \text{TPR}_{i}}{2}
\end{equation}
The AUC quantifies alignment with ground truth, where AUC = 1 indicates perfect alignment and AUC = 0.5 represents random guessing. To normalize:
\begin{equation}
\text{AUC}_{\text{normalized}} = \max(\text{AUC}, 1 - \text{AUC})
\end{equation}

\begin{figure*}[t]
\centering
\includegraphics[width=0.95\textwidth]{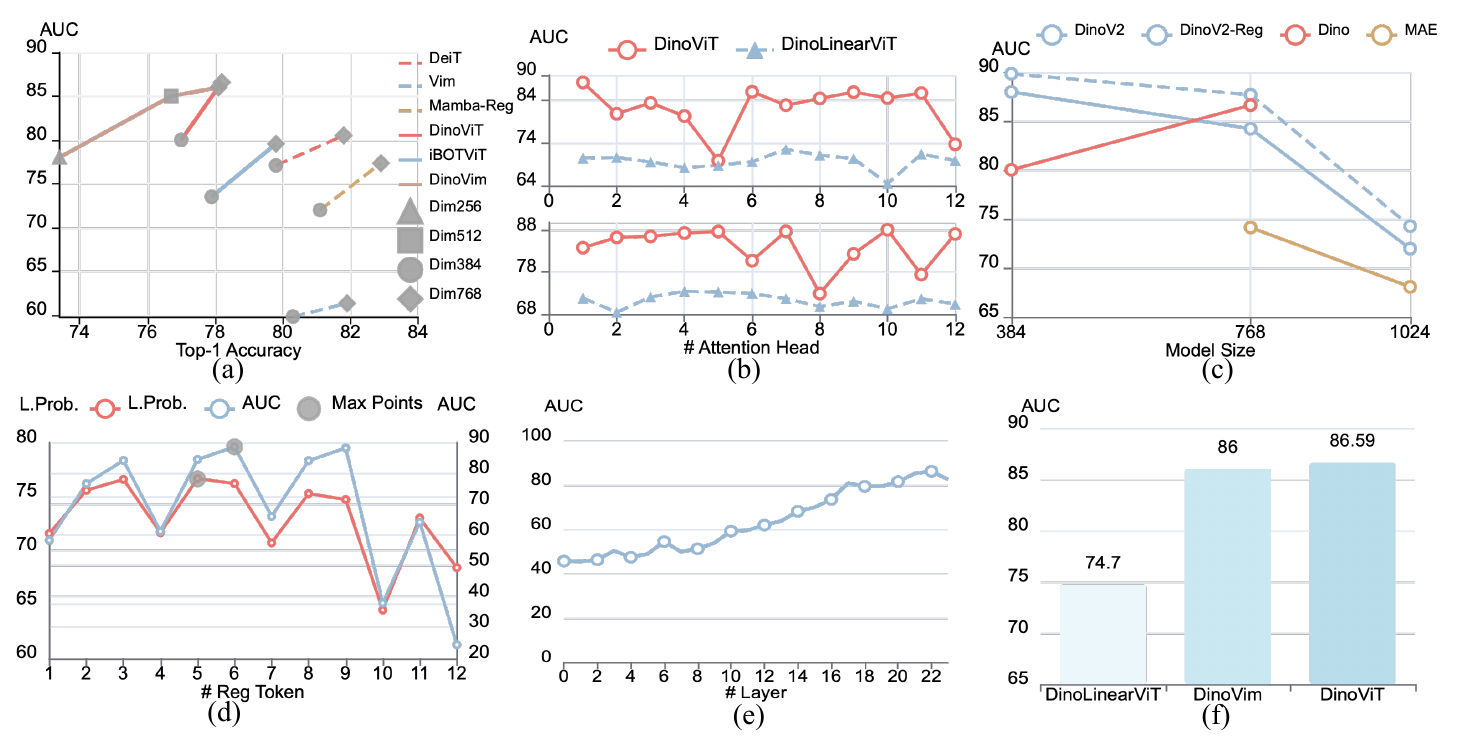}
\caption{\textbf{AUC analysis across different settings} (values in \%): (a) Supervised vs. self-supervised learning and model sizes. (b) Attention head contribution in self/linear attention. (c) Register mechanisms in DINOv2. (d) Register token position effects. (e) Feature quality evolution across layers. (f) AUC comparison: self/linear attention, and Mamba (Base model).}
\label{fig:AUC_all}
\end{figure*}

\subsection{Feature Maps from Self-Supervised Learning}

We choose to pretrain all our models using DINO’s self-supervised learning paradigm. Based on prior experience with ViT feature visualization, models pretrained with DINO often produce the best saliency maps for interpretation. Formally, an input image $\mathbf{I}$ generates augmented crops $\mathbb{I}_{global}$ and $\mathbb{I}_{local}$ for student network $\mathbf{S}$ and teacher network $\mathbf{T}$. The teacher's parameters follow EMA updates from the student. After centering the teacher's output and applying temperature-scaled softmax, we obtain $P_s$ and $P_t$ (gradient-blocked), with cross-entropy loss $-P_t \log P_s$.

\textit{Visual Mamba Models:}
Vim~\cite{vim} and Mamba-Reg~\cite{mamba-r} share bidirectional scanning but differ in register token count. A 2D image $\mathbf{I} \in \mathbb{R}^{H \times W \times N_c}$ is partitioned into flattened patches $\boldsymbol{\mathcal{P}} \in \mathbb{R}^{L \times (N_p^2 \cdot N_c)}$ (where $H$, $W$ are image dimensions, $N_c$ channels, $N_p$ patch size) and linearly projected to dimension $D$. Mamba-Reg introduces $n_r$ register tokens $\boldsymbol{\mathcal{R}}_i \in \mathbb{R}^{1 \times D}$, $i \in \{1, 2, \dots, n_r\}$ (when $n_r = 1$, it reduces to Vim). With positional embeddings $\mathbf{E}_{\text{pos}} \in \mathbb{R}^{(L+n_r) \times D}$, the formulation is:
\begin{equation}
\begin{aligned}
\mathbf{T}_0 &= \left[\boldsymbol{p}^1 \mathbf{W}; \cdots; \boldsymbol{p}^{i-1} \mathbf{W}; \boldsymbol{\mathcal{R}}_i; \right. \\
&\left. \boldsymbol{p}^{i+1} \mathbf{W}; \cdots; \boldsymbol{p}^L \mathbf{W}\right] + \mathbf{E}_{\text{pos}}
\end{aligned}
\end{equation}
Where \( \boldsymbol{p}^j \) represents the \( j \)-th patch of \( \boldsymbol{\mathcal{P}} \), and \( \mathbf{W} \in \mathbb{R}^{(N_p^2 \cdot N_c) \times D} \) is the projection matrix.

Additionally, while vanilla Visual Mamba uses fixed positional encoding for fixed-size images (e.g., 224$\times$224), we adopt DINO's adaptive positional encoding to handle varying image sizes (e.g., 224$\times$224, 96$\times$96) via bicubic interpolation, preserving spatial information across scales.

\section{Experiments}
\label{sec:Experiments}

\subsection{Implementation Details}
Similar to DINO \cite{Dino}, we conduct self-supervised pretraining of the Mamba models on the ImageNet \cite{imagenet} dataset, which includes 1.28 million training images along with 50k validation images, covering 1K categories. Mamba-S/16 model \cite{vim} are trained using the AdamW \cite{adamw} optimizer with a total batch size of 1024 across 16 GPUs, with training employing mixed-precision (BF16) and 10 workers per GPU for data loading. During the first 10 epochs, the learning rate is linearly ramped up to its initial value (lr = 0.0005 * batchsize/256) based on the linear scaling rule, after which it decays according to a cosine schedule \cite{cosineschedule}. Similarly, the weight decay follows a cosine schedule \cite{cosineschedule}, gradually increasing from 0.04 to 0.4. The teacher model's temperature starts at 0.04 and is warmed up to 0.07 over 30 epochs, while the student model’s temperature is fixed at 0.1. We employ bicubic interpolation to adjust the position embeddings, and apply data augmentations of BYOL \cite{BYOL}.

In the Mamba model configuration, we use the Mambav2 \cite{mamba2} model. We train four different model sizes: DinoVim-Tiny, DinoVim-Small, DinoVim-Base, DinoMamba-Reg-Base. Unlike the original Mamba-v1 \cite{mamba}, which has dimension parameters of 192, 384, and 768 for the tiny, small, and base models, Mamba-v2 \cite{mamba2} uses 256, 512, and 768, respectively. All models are trained with 24 layers using a bidirectional scanning approach.
In the Dino LinearViT model configuration, we replace \(\mathrm{softmax}(QK^T)\) with \(\psi(Q) \psi(K)^T\), where \(\psi(\cdot)\) represents the softmax function (i.e., \(\psi(Q) = \mathrm{softmax}(Q)\) and \(\psi(K) = \mathrm{softmax}(K)\)).

\begin{figure*}[t]
\centering
\includegraphics[width=0.9\textwidth]{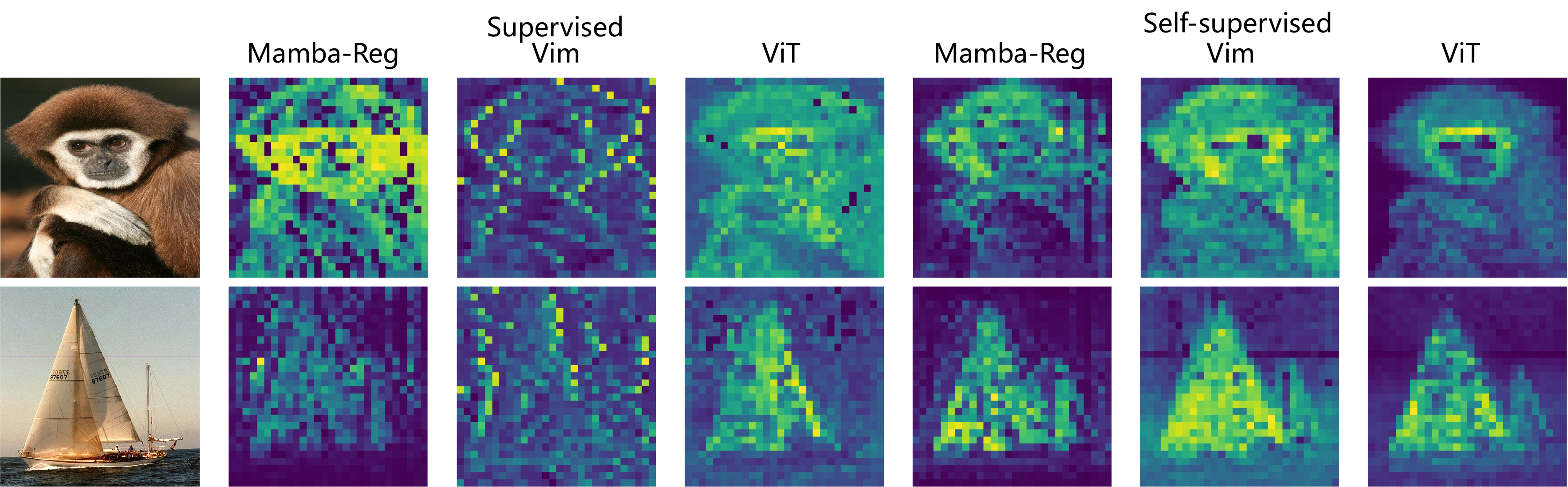}
\caption{\textbf{Comparison of feature map quality in supervised vs. self-supervised settings.} Feature maps from Mamba-Reg, Vim, and ViT models trained with supervised learning (left) and DINO self-supervised learning (right) on two example images. Self-supervised learning produces significantly clearer maps with better foreground-background distinction and reduced noise. ViT achieves the cleanest activations, while Mamba-based models show comparable quality in the self-supervised setting.}
\label{fig:Vis-super-unsuper}
\end{figure*}

\subsection{Visualization-Based Analysis}

\paragraph{Comparison in Supervised \& Self-Supervised Learning} 
Figure~\ref{fig:AUC_all} (a) quantifies feature map quality across different backbones \cite{deit,vim,mamba-r,Dino} in supervised and self-supervised settings. Dashed and solid lines represent AUC scores for supervised and self-supervised learning, respectively. Self-supervised methods consistently achieve higher AUC scores, indicating clearer feature maps in the DINO framework (Figure \ref{fig:Vis-super-unsuper}). This highlights the advantage of analyzing feature maps in self-supervised learning.

\paragraph{Comparison Across Different Model Sizes}
Figure~\ref{fig:AUC_all} (a) shows the relationship between model size and feature map quality. Larger models achieve higher AUC scores, suggesting that increased capacity enhances feature representation. Additionally, Figure~\ref{fig:combined} (Bottom) presents a comparison of Mamba’s feature maps across different model sizes.

\paragraph{Evaluating Attention Head Contribution}
Figure~\ref{fig:AUC_all}(b) shows AUC scores of 12 attention heads in the last two layers. Some heads achieve high scores for effective feature extraction, while others exhibit "lazy behavior." This unequal distribution reveals certain heads focus on less relevant information.  Linear Attention yields lower AUC scores than Self-Attention, aligning with its lower matrix rank and reduced representational capacity. These results validate AUC as a meaningful metric for evaluating feature maps and analyzing information distribution across attention heads.

\paragraph{Analyzing in DINOv2 and the Impact of Register }
Figure \ref{fig:AUC_all} (c) analyzes results from \cite{Visiontransformersneedregisters}, which show that adding registers significantly improves feature map quality in DINOv2 \cite{dinov2} by reducing high-norm outliers and enhancing model performance. Interestingly, while AUC generally increases with model size, DINOv2 exhibits the opposite trend. This aligns with \cite{Visiontransformersneedregisters}, which found that larger models produce more outliers, degrading feature quality. Further analysis suggests this decline is linked to the masked patch strategy in MAE \cite{mae}.

\paragraph{Analysis of Register Token Position Impact} Figure \ref{fig:AUC_all} (d) examines the impact of register token positions using linear probing and AUC scores. With 12 evenly distributed tokens, central positions capture more information, resulting in clearer feature maps and higher AUC scores. This finding is consistent with existing work \cite{vim}, where the middle CLS token demonstrates the best performance. Additionally, the strong correlation between AUC scores and linear probing accuracy suggests that feature map quality serves as a proxy for model performance. Both metrics exhibit a wave-like pattern, reflecting periodic influences on information acquisition. Note: AUC values in this figure are before normalization.

\paragraph{Evolution of Feature Quality Across Network Layers}
Figure \ref{fig:AUC_all} (e) shows AUC scores across layers, revealing how feature quality evolves with depth. Early layers have AUC below 50\%, indicating a focus on background information. As layers deepen, AUC rises, reflecting a shift toward target-specific features and more meaningful representations. Note: AUC values are before normalization.

\subsection{Quantitative Results}
In this section, to ensure a fair comparison, we keep the parameter count identical to the original model and make only minimal changes to the training setup, following ViT-B as reference: increasing gradient clipping (0.3 → 3.0), enabling BF16, and reducing the minimum learning rate (2e-6 → 1e-6), while keeping all other settings unchanged.

\begin{table}[h!]
\centering
\tablestyle{5pt}{1.1}
\begin{tabular}{l|cccc}
\multirow{2}{*}{Architecture} & 
\multirow{2}{*}{Dimension} & 
\multicolumn{2}{c}{\#Param. (M)} & 
\multirow{2}{*}{Top-1 acc.} \\
& & Backbone & Head & \\
\shline
ViT-S & 384 & 21 & 1.5 & 77.0 \\
ViT-B & 768 & 85 & 1.5 & \bf 78.2 \\
LinearViT-B & 768 & 85 & 1.5 & 74.7 \\
DinoVim-S & 512 & 40 & 1.5 & 76.7 \\
DinoVim-B & 768 & 88 & 1.5 & 78.1 \\
\hline
\multicolumn{5}{l}{\it \textcolor{gray}{Models with heavy-weight probing heads:}} \\
DinoVim-T & 256 & 10 & 2.0 & 73.7 \\
DinoVim-S & 512 & 40 & 2.0 & 77.4 \\
DinoVim-B & 768 & 88 & 6.2 & 78.3 \\\rowcolor{orange!20}
DinoMa.-R.-B & 768 & 88 & 6.2 & \bf 78.5 \\
\end{tabular}
\caption{\textbf{Linear probing results on ImageNet-1k.} All models are pretrained with the DINO paradigm~\cite{Dino}. ``Ma.-R.-B'' denotes MambaReg-Base~\cite{mamba-r}. Best results for each setup are \textbf{bolded}.}
\label{tab:ImageClassification}
\end{table}

\begin{figure}[!h]
\centering
\includegraphics[width=0.9\columnwidth]{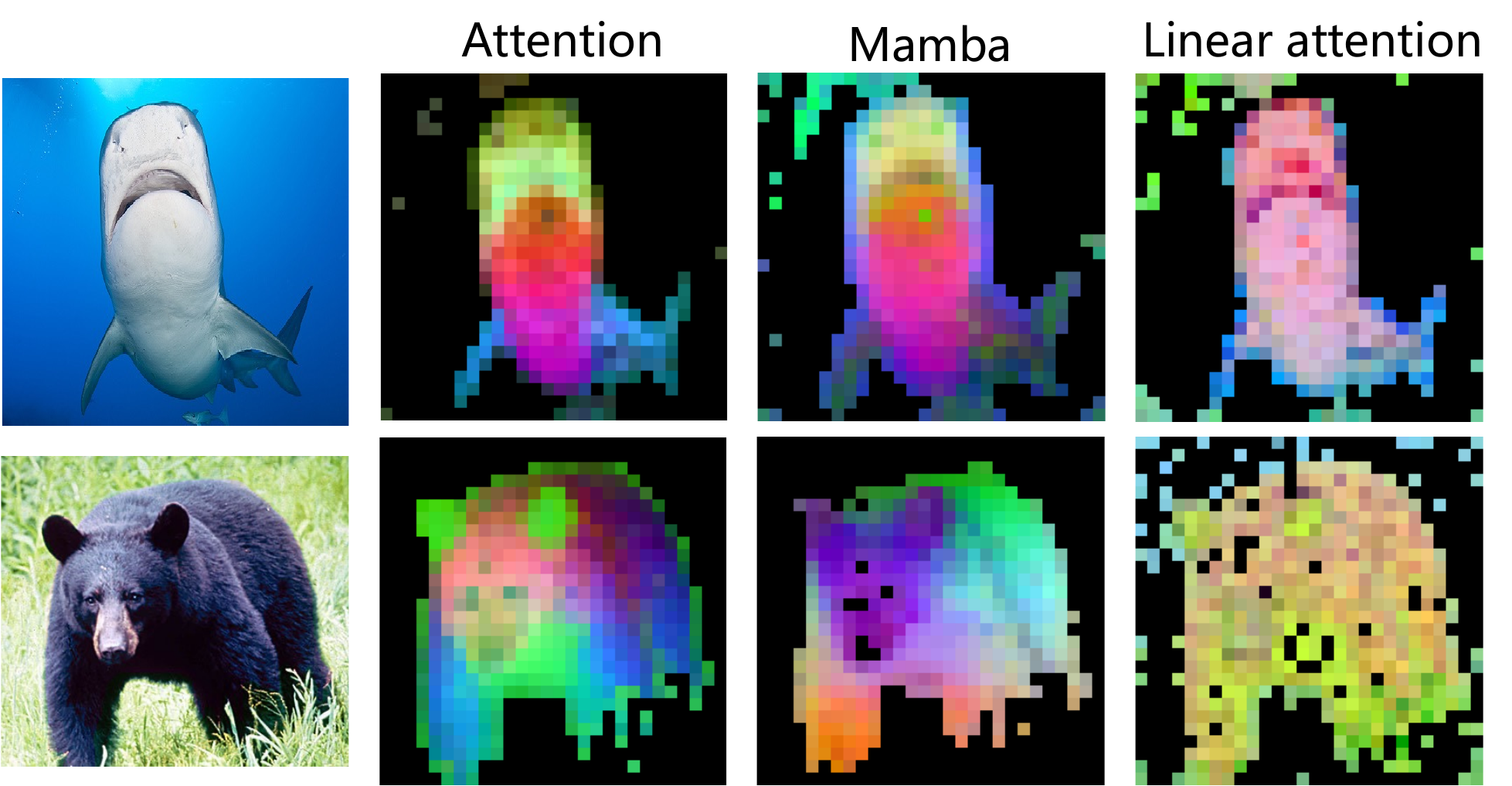}\\

\includegraphics[width=0.9\columnwidth]{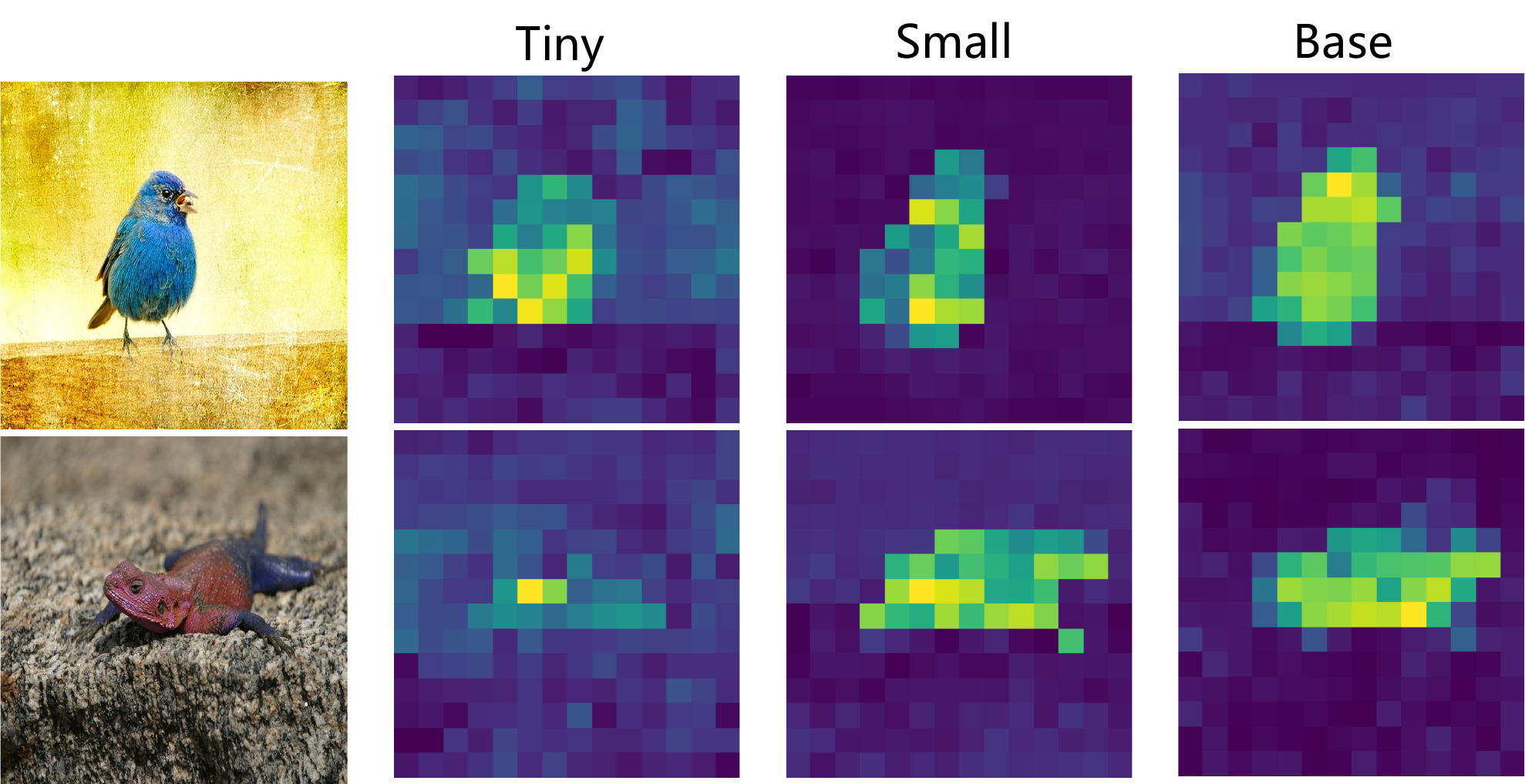}
\caption{\textbf{ PCA and model size comparison.} (Top) PCA visualization of self-attention, Mamba, and linear attention. (Bottom) Feature map comparison across different DinoVim model sizes.}
\label{fig:combined}
\end{figure}

\begin{table}[h!]
\centering
\tablestyle{4.5pt}{1.1}
\begin{tabular}{l|cccc}
Backbone & Mixer & Image size & mIoU (\%) & mAcc (\%) \\
\shline
LinearViT-B & linear attn. & 512$\times$512 & 29.2 & 38.7 \\
DinoVim-B & mamba-2 & 512$\times$512 & 38.0 & 47.1 \\\rowcolor{orange!20}
ViT-B & self-attn. & 512$\times$512 & \bf 43.2 & \bf 52.5 \\
\end{tabular}
\caption{\textbf{Semantic segmentation performance on ADE20K.} UperNet with different DINO-pretrained backbones. ViT-B outperforms DinoVim-B, which surpasses LinearViT-B, reflecting their respective representational capacities.}
\label{tab:SemanticSegmentation}
\end{table}

\begin{table}[h!]
\centering
\tablestyle{5.5pt}{1.2}\begin{tabular}{l|cccccc}
Backbone & AP$^{\text{b}}$ & AP$^{\text{b}}_{50}$ & AP$^{\text{b}}_{75}$ & AP$^{\text{b}}_{\text{s}}$ & AP$^{\text{b}}_{\text{m}}$ & AP$^{\text{b}}_{\text{l}}$ \\
\shline
LinearViT-B & 37.1 & 54.7 & 39.9 & 20.3 & 39.8 & 50.3 \\
DinoVim-B & 42.8 & 60.3 & 46.4 & 25.4 & 46.0 & 57.0 \\\rowcolor{orange!20}
ViT-B & 44.8 & 63.7 & 48.4 & 25.8 & 47.8 & 61.0 \\

Backbone & AP$^{\text{m}}$ & AP$^{\text{m}}_{50}$ & AP$^{\text{m}}_{75}$ & AP$^{\text{m}}_{\text{s}}$ & AP$^{\text{m}}_{\text{m}}$ & AP$^{\text{m}}_{\text{l}}$ \\
\hline
LinearViT-B & 32.6 & 52.1 & 34.8 & 15.0 & 34.5 & 48.9 \\
DinoVim-B & 37.4 & 58.0 & 40.3 & 18.2 & 40.1 & 55.0 \\\rowcolor{orange!20}
ViT-B & 39.1 & 61.0 & 41.9 & 18.9 & 41.0 & 58.1 \\
\end{tabular}
\caption{\textbf{Detection and segmentation results on COCO 2017}. Performance with Cascade Mask R-CNN using DINO-pretrained ViT-B, DinoVim-B, and LinearViT-B backbones.}
\label{tab:ObjectDetectionInstanceSegmentation}
\end{table}

\begin{table}[!h]
\centering
\tablestyle{4.5pt}{1.1}
\begin{tabular}{l|ccccccc}

Model & MF & T0.7 & TI & Sketch & A & R & Real \\
\shline
LinearViT-B & 62.5 & 71.0 & 76.9 & 21.6 &  9.6 & 32.7 & 81.3 \\
DinoVim-B        & 66.6 & 75.0 & 80.2 & 27.6 & 14.2 & 33.1 & 84.3 \\
\rowcolor{orange!20}
ViT-B        & 66.6 & 75.4 & 80.7 & 25.5 & 15.4 & 38.0 & 84.6 \\

\end{tabular}
\caption{\textbf{Robustness evaluation on ImageNet variants.} Top-1 accuracy comparison on ImageNet-V2 (MF: Matched-Frequency, T0.7: Threshold 0.7, TI: Top-Images) and ImageNet-(Sketch, A, R, Real).}
\label{tab:imagenet-comparison}
\end{table}

\paragraph{Image Classification}  
As shown in Table \ref{tab:ImageClassification}, for linear probing, we find that Self-Attention outperforms Mamba, which in turn outperforms Linear Attention. While Mamba achieves performance close to ViT, it requires more parameters. This trend is consistent with the AUC scores presented in Figure \ref{fig:AUC_all} (f), where Self-Attention achieves the highest AUC scores, indicating the strongest feature representation capability. Mamba  follows closely with a slightly lower AUC scores. In contrast, Linear Attention has the lowest AUC scores, reflecting its limited representational capacity due to its lower matrix rank. Figure \ref{fig:combined} (Top) shows that self-attention clearly separates features, while Mamba is more blurred, and linear attention lacks differentiation.

\paragraph{Semantic Segmentation} To evaluate features learned under the DINO paradigm, we initialize the semantic segmentation model with weights pretrained using DINO on the image classification task. As presented in Table \ref{tab:SemanticSegmentation}, Self-Attention achieves better segmentation performance than Linear Attention. Mamba further outperforms LinearViT-B, demonstrating its advantage in long-sequence modeling.

\paragraph{Object Detection and Instance Segmentation} As shown in Table \ref{tab:ObjectDetectionInstanceSegmentation}, Self-Attention outperforms Mamba in object detection and segmentation, reflecting its higher matrix rank and stronger feature representation. In contrast, Mamba surpasses LinearViT-B, with its state-space model enabling more efficient long-sequence processing and better global information utilization, making it competitive in long-range modeling tasks. LinearViT-B performs worse due to its lower matrix rank and weaker spatial representation.

\paragraph{Robustness Evaluation} 
We evaluate model robustness on ImageNet-V2~\citep{recht2019imagenet}, ImageNet-Sketch~\citep{wang2019learning}, ImageNet-A~\citep{hendrycks2021natural}, ImageNet-R~\citep{hendrycks2021many}, and ImageNet-Real~\citep{beyer2020we}. Table~\ref{tab:imagenet-comparison} shows that ViT achieves the highest accuracy under these dataset distribution shifts, followed by Mamba, while LinearViT exhibits a noticeable performance drop, consistent with our theoretical rank analysis. These results suggest that higher-rank architectures generalize better to unseen distributions, with Mamba maintaining competitive robustness despite its linear complexity.
\section{Conclusion}
\label{sec:Conclusion}

This paper investigates the representational capacity of Vision Mamba through rank analysis and feature map visualization, demonstrating its balance between efficiency and expressiveness. Experiments show that Mamba matches Transformers in classification and outperforms Linear Attention in high-resolution segmentation and detection, highlighting its superior long-sequence modeling. To assess feature quality, we propose the Binary-AUC Metric, which also serves as a stethoscope for model analysis, identifying underperforming components (e.g., attention heads, register tokens). Improving their effectiveness could further enhance the representational capacity of both Attention and Mamba.

{
    \small
    \bibliographystyle{ieeenat_fullname}
    \bibliography{main}
}

\clearpage
\setcounter{page}{1}
\maketitlesupplementary

\section*{Appendix}

\subsection*{A. More Feature Map Visualization}
Due to space limitations in the main text, this section provides supplementary visualizations of the feature maps for the Transformer \cite{attention}/Mamba \cite{mamba2} models presented in the main text, as shown in Figure \ref{fig:supp-vis-mid-max-mean}, \ref{fig:supp-vis-model-size}, \ref{fig:supp-Vis-super-unsuper}, \ref{fig:Supp-vIs-linear-mamba-attn}, \ref{fig:supp-PCA}.

\subsection*{B. Visualizing Long-Range Dependencies}
In this section, utilizing Mamba-Reg \cite{mamba-r}, we discovered that inserting register tokens at different positions allows Mamba to exhibit a multi-head-like characteristic similar to Transformers, as shown in Figure \ref{fig:supp-vis-mamba-reg}. Each register token functions akin to an attention head, with even the forward-placed register tokens (e.g. Image 2, Index 2) able to perceive distant target objects within the image. This demonstrates Mamba's ability to capture global information, transcending the limitations of local token interactions. 

\begin{figure*}[!h]
    \centering
    \includegraphics[width=0.9\linewidth]{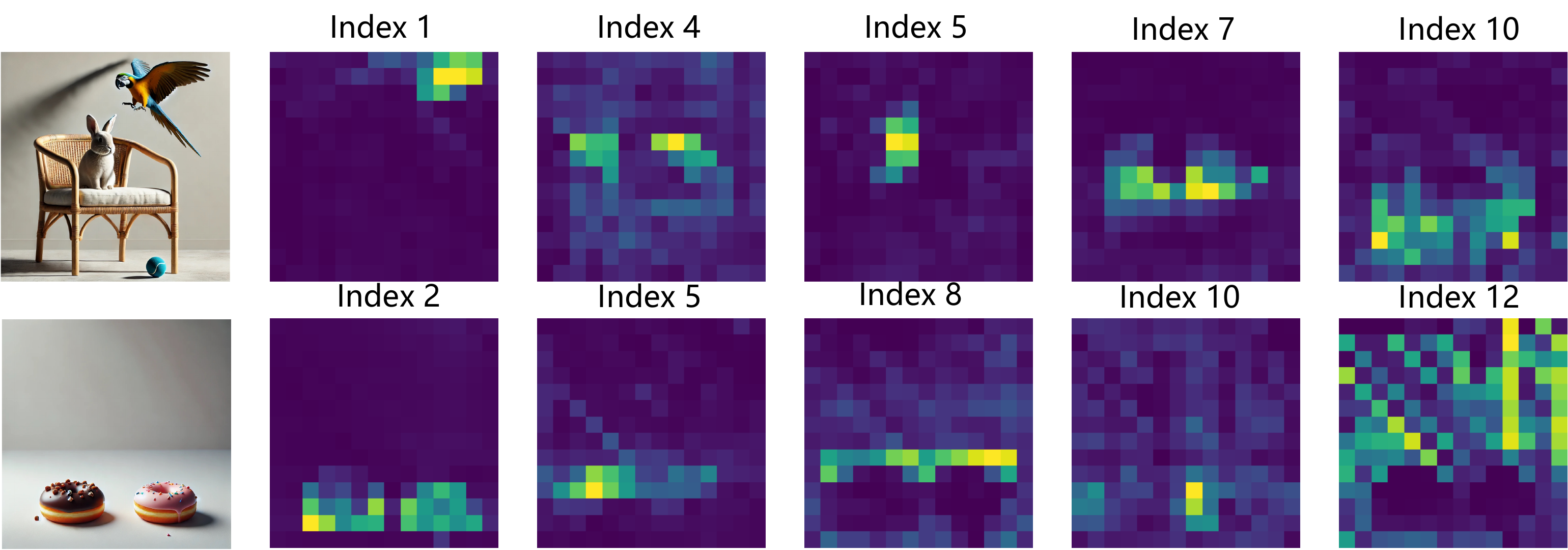}
    \captionsetup{type=figure, width=0.9\textwidth}
    \vspace{-0.2cm}
    \caption{Feature maps corresponding to different register tokens, which are evenly distributed among sequence tokens, reveal a role similar to multi-head attention. Forward-placed register tokens (e.g., Image 2, Index 2) capture global patterns, while later tokens (e.g., Index 10) focus on specific regions, demonstrating Mamba's balance of global and local information.}
    \label{fig:supp-vis-mamba-reg}
    \vspace{-0.3cm}
\end{figure*}

\begin{figure*}[!h]
    \centering
    \begin{minipage}{0.49\linewidth}
        \centering
        \includegraphics[width=\linewidth]{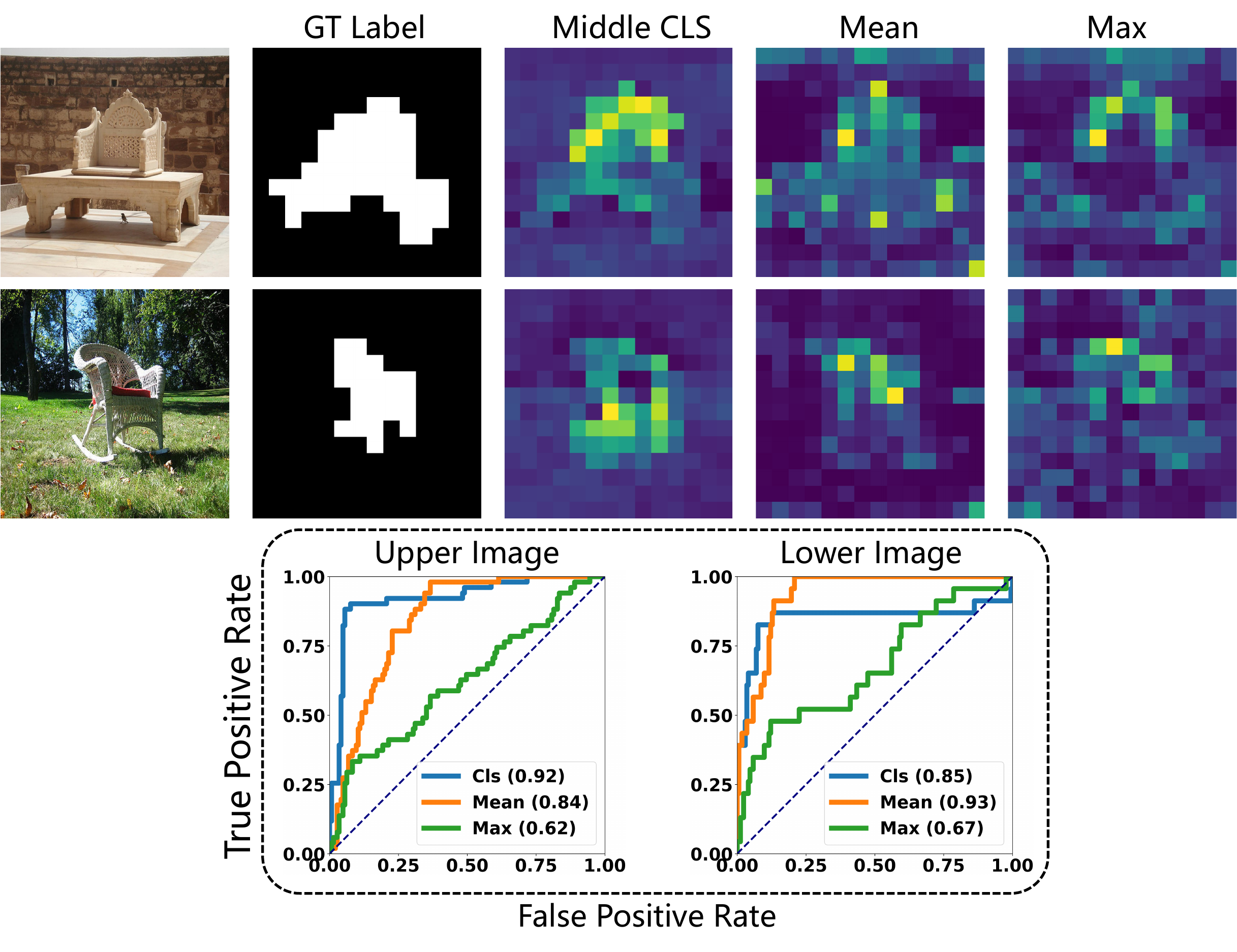}
       
        \caption{Feature map comparison in DinoVim: The Middle cls captures global information with clearer feature maps. Bottom: ROC curve comparisons.}
        \label{fig:supp-vis-mid-max-mean}

    \end{minipage}
    \hfill
    \begin{minipage}{0.49\linewidth}
        \centering
        \includegraphics[width=\linewidth]{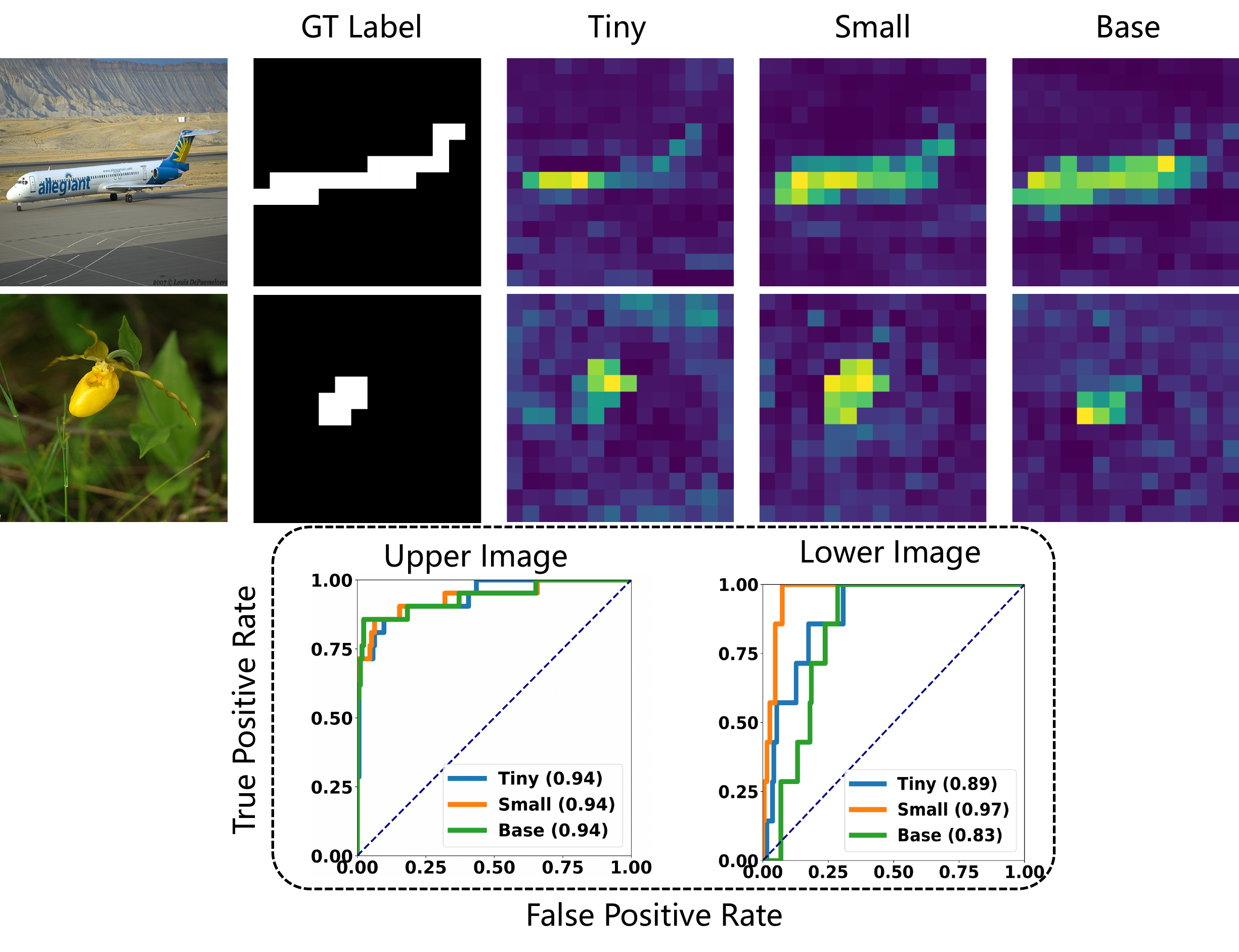}

        \caption{Feature map comparison in DinoVim across different model sizes, the small model displays the clearest feature maps. Bottom: ROC curve comparisons.}
        \label{fig:supp-vis-model-size}

    \end{minipage}
    
\end{figure*}

\begin{figure*}[!h]
    \centering
    \includegraphics[width=0.75\linewidth]{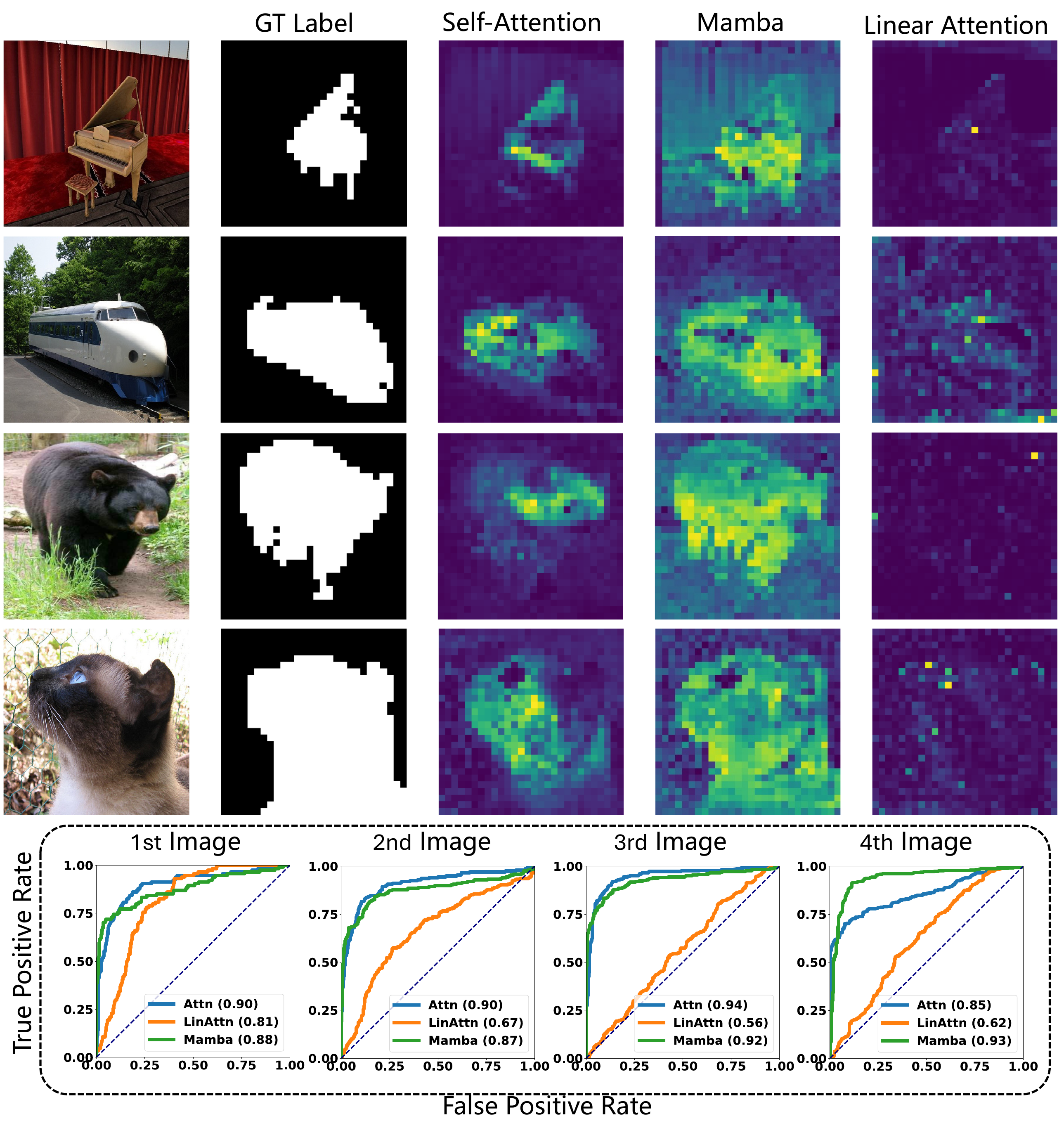}
    \parbox{0.75\linewidth}{
        \captionof{figure}{Feature maps comparison among Self-Attention, Mamba, and Linear Attention shows that Self-Attention provides the clearest target-background distinction, Mamba captures moderate clarity with global patterns, while Linear Attention exhibits blurred boundaries and weaker separation.}
        \label{fig:Supp-vIs-linear-mamba-attn}
    }
\end{figure*}

\begin{figure*}[!h]
    \centering
    \includegraphics[width=1.0\linewidth]{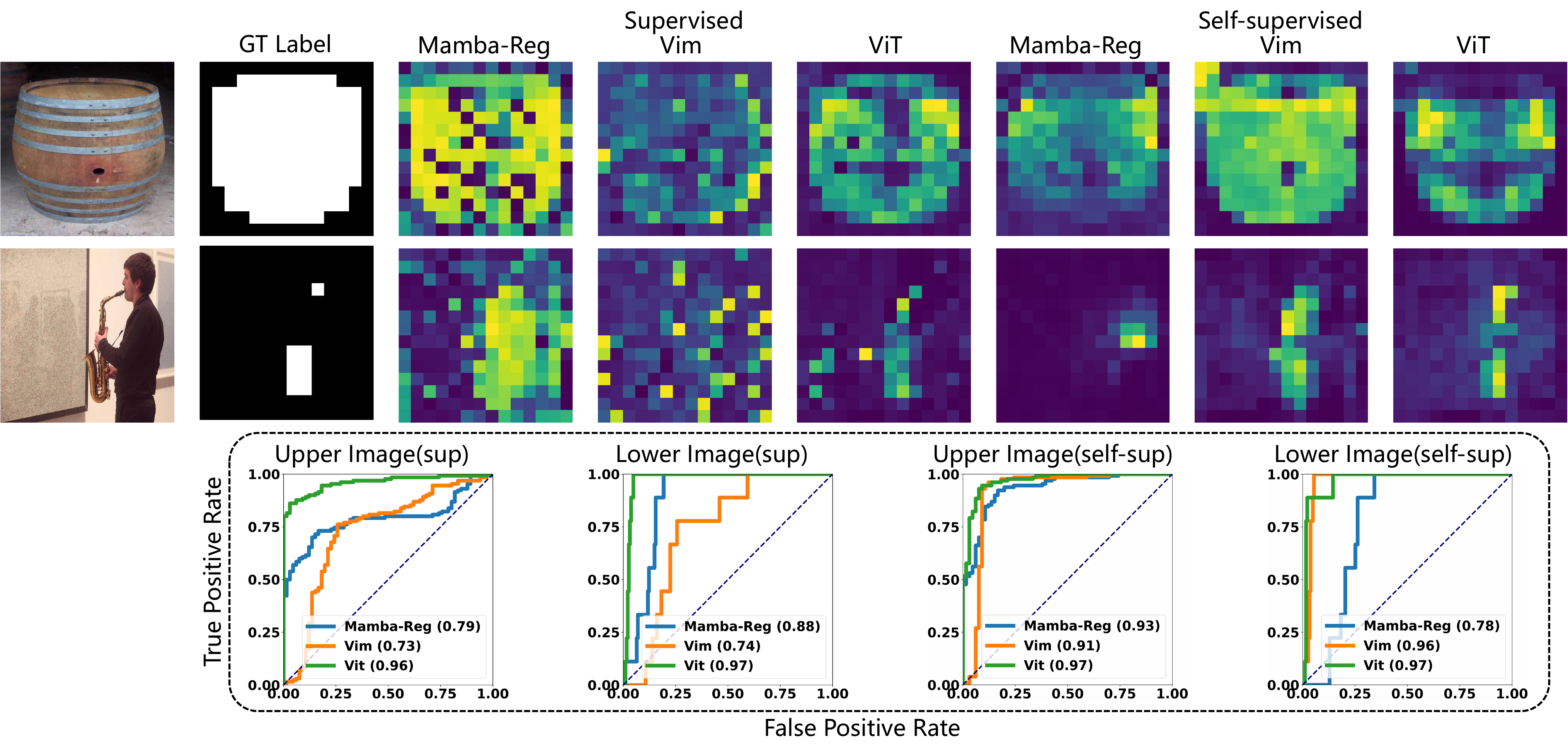}
    \caption{Feature map quality comparison: Supervised vs. Self-supervised. The figure illustrates that feature maps generated by self-supervised learning are clearer compared to those from supervised learning. Additionally, ViT produces less noisy feature maps compared to Mamba. The bottom section presents the ROC curve comparisons.}
    \label{fig:supp-Vis-super-unsuper}
\end{figure*}

\begin{figure*}[!h]
    \centering
    \includegraphics[width=1.0\linewidth]{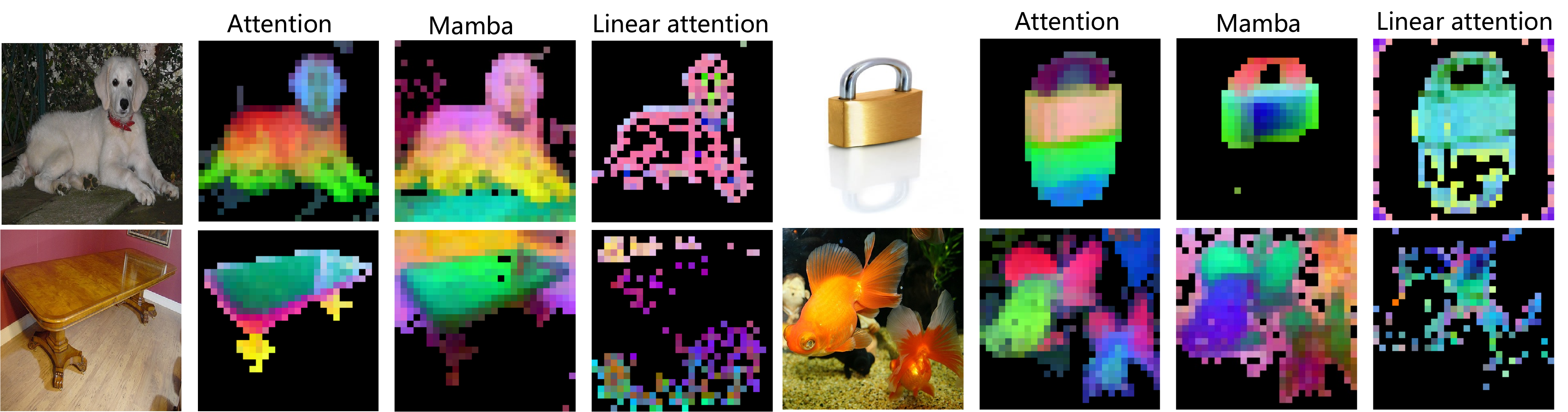}
    \caption{More visualization of the first PCA components of self-attention, Mamba, and linear attention, all generated under a consistent threshold.}
    \label{fig:supp-PCA}
\end{figure*}

\textbf{Evaluation dataset for feature maps}
To accurately evaluate the quality of the model's performance in visualization, we use ImageNet-S \cite{imagenet-s} as the evaluation dataset. This dataset contains 919 categories, with 1,183,322 images for training, 12,419 for validation, and 27,423 for testing. Each image is accompanied by high-quality semantic segmentation annotations, enabling a detailed analysis of the quality and interpretability of feature maps. In our setting, we use only the validation set for evaluation.

\subsection*{C. Downstream tasks setting} 
\label{sec:appendix-Downstreamtaskssetting}
To further investigate the performance of Mamba as a pre-trained model in the Dino setting, we conduct two downstream experiments: Semantic Segmentation, and Object Detection and Instance Segmentation.

\textbf{Semantic Segmentation Settings.} We conduct the semantic segmentation task on the ADE20K dataset \cite{ade20k}, which comprises 20K training images, 2K validation images, and 3K test images, encompassing 150 fine-grained semantic categories. For our experiments, we utilize the UperNet framework \cite{UperNet} as the baseline. The optimizer is configured as AdamW \cite{adamw} with a learning rate of 6e-5, momentum parameters of (0.9, 0.999), and a weight decay of 0.05. The learning rate scheduling includes two phases: initially, LinearLR is applied for the first 1500 steps with a starting factor of 1e-6. Following this, the main training phase utilizes a PolyLR scheduler up to 160,000 steps. 

\textbf{Object Detection and Instance Segmentation Settings} We perform object detection and instance segmentation experiments on the COCO 2017 dataset \cite{coco2017}, which includes 118K training, 5K validation, and 20K test images. Cascade Mask R-CNN \cite{cascadercnn} is adopted as the base framework. The optimization setup uses AdamW \cite{adamw} with a learning rate of 0.0001 and a weight decay of 0.1, where specific parameters, such as biases and positional embeddings, have decay multipliers set to zero to avoid excess regularization. The learning rate schedule comprises two phases: initially, LinearLR is applied for the first 500 steps to gradually increase the learning rate, followed by MultiStepLR, which decays the learning rate at epochs 8 and 11, ensuring stable convergence during training. The images are resized to \(1333 \times 800\) pixels for training, validation, and testing, ensuring consistency across these stages.

\subsection*{D. More visualization of PCA}
This figure \ref{fig:supp-PCA} presents the PCA visualization of different attention mechanisms, where color distribution reflects how each model captures features in the representation space. The PCA results of self-attention clearly distinguish different parts of the object, providing a well-defined separation of features. In contrast, Mamba appears slightly more blurred than self-attention, with lower differentiation across object regions. Linear attention, on the other hand, exhibits a more sparse PCA representation and struggles to effectively separate different parts using color. This suggests that each mechanism focuses on different aspects of feature extraction, and the PCA visualization further highlights their distinct information distributions.

\begin{figure*}[t]
    \centering
    \includegraphics[width=1.0\linewidth]{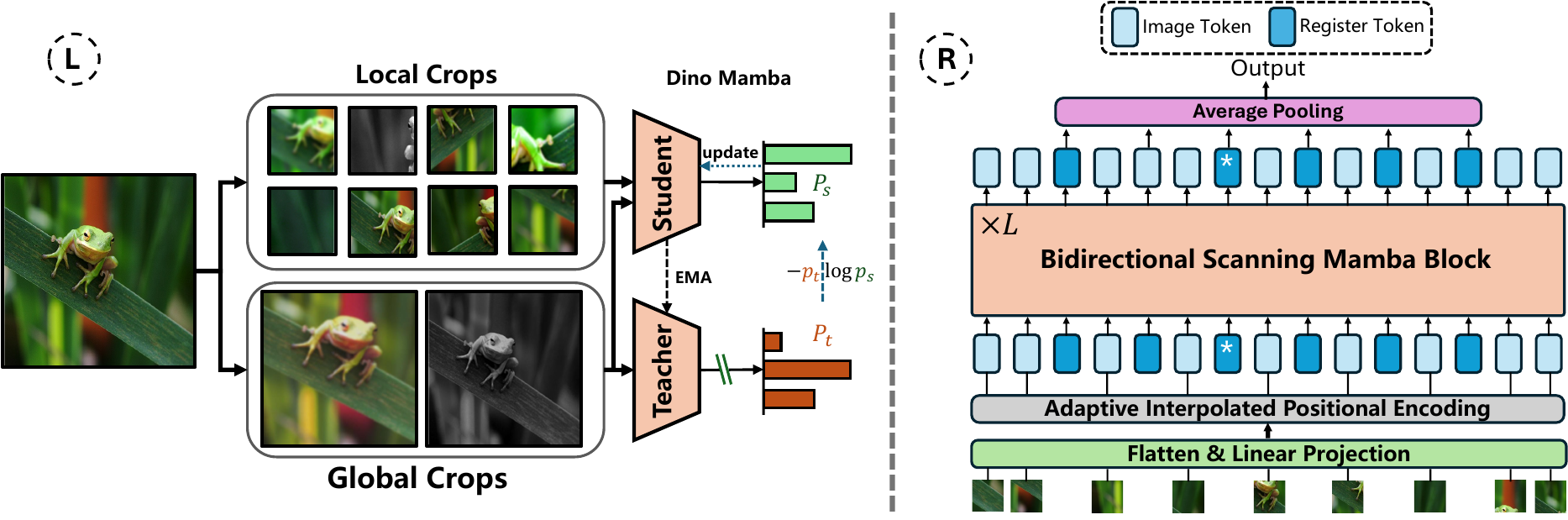} 
    \vspace{-0.2cm}
    \caption{Visual Mamba in the DINO Framework: The \textbf{left} side shows the standard DINO architecture. The \textbf{right} side illustrates the structure of the student and teacher models, inspired by Mamba-Reg \cite{mamba-r}, with register tokens inserted at various positions to observe learned representations. When only a single register token is used at the middle position (`*'), this setup corresponds to the Vim \cite{vim} model. Additionally, an adaptive iterpolated positional encoding module is included to accommodate inputs of varying sizes (e.g., 224$\times$224, 96$\times$96). }
    \label{fig:mambaDINOFramework}
    \vspace{-0.2cm}
\end{figure*}

\subsection*{E. Visual Mamba in the DINO Framework}
This section supplements the previous discussion with additional details and visualization. As shown in Figure \ref{fig:mambaDINOFramework}, the model employs bidirectional scanning with register tokens for feature aggregation. Vim uses a single register token, while Mamba-Reg introduces multiple tokens for hierarchical learning. The visualization highlights how positional encoding and token interactions enhance representation learning.

\subsection*{F. Hyper-parameter ablation}

In the ablation study, we systematically evaluate various design choices through controlled experiments on the small DinoVim model.

\textbf{Latent feature expansion for linear probing.} In linear probing protocols, a common way to improve predictive performance is to concatenate the output of multiple blocks as an expanded latent feature. In our experiments, by default we follow DINO~\cite{Dino}'s approach to combine the output of the last four model blocks, and we summarize the ablation results of this mechanism in the left section of Table~\ref{tab:parallel_tables}. As shown, the top-1 accuracy improves with more concatenated blocks, peaking at 77.2\% with 4 blocks.

\begin{table}[h!]
    \centering
    \small 
    \begin{tabular}{c@{\hspace{0.15cm}}c@{\hspace{0.15cm}}c} 
        \begin{tabular}{c|c}
            \toprule
            \textbf{\#blocks} & \textbf{acc} \\
            \midrule
            1 & 71.8 \\
            2 & 75.6 \\
            3 & 76.7 \\
            \textbf{4} & \textbf{77.2} \\
            \bottomrule
        \end{tabular}
        &
        
        \begin{tabular}{c|c}
            \toprule
            \textbf{wd} & \textbf{acc} \\
            \midrule
            0 & 77.2 \\
            1e-2 & 75.7 \\
            5e-3 & 76.4 \\
            \textbf{1e-3} & \textbf{77.3} \\
            \bottomrule
        \end{tabular}
        &

        \begin{tabular}{c|c}
            \toprule
            \textbf{lr} & \textbf{acc} \\
            \midrule
            2.5e-4 & 77.1 \\
            5e-4 & 77.2 \\
            1e-3 & 77.3 \\
            \textbf{2e-3} & \textbf{77.4} \\
            \bottomrule
        \end{tabular} \\
    \end{tabular}
    \caption{Comparison of top-1 accuracy for different hyperparameters: \textbf{Left}: Number of blocks. \textbf{Middle}: Weight decay (fixed lr = 1e-3). \textbf{Right}: Learning rate (fixed wd = 1e-3).}
    \label{tab:parallel_tables}
    
\end{table}

\noindent\textbf{Learning rate and weight decay.} As shown in the middle/right section of Table~\ref{tab:parallel_tables}, top-1 accuracy varies with different \(\texttt{lr}\) and \(\texttt{wd}\) combinations. Fixing \(\texttt{lr} = 1\text{e-3}\), increasing \(\texttt{wd}\) from \(0\) to \(1\text{e-3}\) raises accuracy from \(77.2\) to \(77.3\). Likewise, with \(\texttt{wd} = 1\text{e-3}\), increasing \(\texttt{lr}\) from \(2.5\text{e-4}\) to \(2\text{e-3}\) achieves the highest accuracy of \(77.4\). The optimal setting (\(\texttt{lr} = 2\text{e-3}, \texttt{wd} = 1\text{e-3}\)) yields the best \(77.4\%\).

\begin{wraptable}{r}{0.25\textwidth}
    \centering
    \small
    \renewcommand{\arraystretch}{0.95}
    \hspace{-5mm} 
    \begin{tabular}{c|c}
        \toprule
        \textbf{Cls strat.} & \textbf{Top-1 acc} \\
        \midrule
        mean & 75.9 \\
        max & 74.6 \\
        middle cls, mean & 76.5 \\
        middle cls, max & 76.3 \\
        \textbf{middle cls} & \textbf{77.2} \\
        \bottomrule
    \end{tabular}
    \caption{Top-1 accuracy of different classification strategies.}
    \vspace{-0.5cm}
    \label{tab:cls_strategy_accuracy}
\end{wraptable}

\textbf{Ablation of Classification Design}  
As shown in Table \ref{tab:cls_strategy_accuracy}, the \textit{middle cls} strategy achieves the highest top-1 accuracy at 77.2, outperforming the \textit{mean} (75.9) and \textit{max} (74.6) strategies. Combined strategies, such as \textit{middle cls, mean} (76.5) and \textit{middle cls, max} (76.3), offer slight improvements over \textit{mean} or \textit{max} alone but fall short of the \textit{middle cls}'s standalone performance. This indicates that the \textit{middle cls} alone possesses sufficient representational capacity to effectively summarize the feature map.

\subsection*{G. Mathematical Foundations for Rank Analysis}

This appendix provides the formal mathematical foundations that underpin the rank-based analysis presented in Section \ref{sec:TheoreticalAnalysis} of the main paper. We present three well-established lemmas from matrix theory regarding block lower triangular matrices, Hadamard products, and matrix products. We provide their proofs for completeness and demonstrate how they directly support our theoretical analysis.

The following lemma establishes a fundamental lower bound on the rank of block lower triangular matrices, which is essential for our analysis of the matrix $\mathbf{M}$ in Equation \ref{Unified-f}. This is a well-known result in matrix theory.
\textbf{Note:} The following lemmas are written using independent notations.

\subsection*{Hadamard Product Rank Bounds}

The following lemma establishes rank bounds for the Hadamard (element-wise) product of matrices, which is crucial for analyzing the structure of $\mathbf{M}$ in our unified formulation. This bound is a standard result in matrix analysis.

\paragraph{Rank of Matrix Products}

The following lemma establishes a fundamental upper bound on the rank of matrix products, which is essential for analyzing the representational capacity of composed transformations in our models. This is one of the most important results in linear algebra for understanding how information flows through sequential operations. \textbf{Note}: \textit{The following lemma uses independent abstract notations for clarity}

\begin{lemma}[Rank Bound for Matrix Products]
\label{lem:matrix_product_rank}
For any two matrices $\mathbf{A} \in \mathbb{R}^{m \times n}$ and $\mathbf{B} \in \mathbb{R}^{n \times p}$, the rank of their product satisfies:
\[
\operatorname{rank}(\mathbf{AB}) \le \min\{\operatorname{rank}(\mathbf{A}), \operatorname{rank}(\mathbf{B})\}.
\]
\end{lemma}

\begin{proof}
We prove both inequalities separately.

\textit{Part 1}: $\operatorname{rank}(\mathbf{AB}) \le \operatorname{rank}(\mathbf{B})$

Let $r_{\mathbf{B}} = \operatorname{rank}(\mathbf{B})$. Then the column space of $\mathbf{B}$ has dimension $r_{\mathbf{B}}$, which means every column of $\mathbf{B}$ can be expressed as a linear combination of $r_{\mathbf{B}}$ basis vectors.

The columns of $\mathbf{AB}$ are linear combinations of the columns of $\mathbf{A}$, where the coefficients come from the columns of $\mathbf{B}$. More precisely, if $\mathbf{b}_j$ denotes the $j$-th column of $\mathbf{B}$ and $\mathbf{a}_i$ denotes the $i$-th column of $\mathbf{A}$, then the $j$-th column of $\mathbf{AB}$ is:
\[
(\mathbf{AB})_j = \sum_{i=1}^{n} b_{ij} \mathbf{a}_i.
\]

Since each column of $\mathbf{B}$ lies in a space of dimension $r_{\mathbf{B}}$, the columns of $\mathbf{AB}$ must lie in a space of dimension at most $r_{\mathbf{B}}$. Therefore:
\[
\operatorname{rank}(\mathbf{AB}) \le \operatorname{rank}(\mathbf{B}).
\]

\textit{Part 2: $\operatorname{rank}(\mathbf{AB}) \le \operatorname{rank}(\mathbf{A})$}

Consider the row space perspective. Let $r_{\mathbf{A}} = \operatorname{rank}(\mathbf{A})$. The row space of $\mathbf{A}$ has dimension $r_{\mathbf{A}}$.

Each row of $\mathbf{AB}$ is a linear combination of the rows of $\mathbf{B}$, where the coefficients come from the corresponding row of $\mathbf{A}$. Specifically, if $\mathbf{a}_i^{\top}$ denotes the $i$-th row of $\mathbf{A}$ and $\mathbf{b}_j^{\top}$ denotes the $j$-th row of $\mathbf{B}$, then the $i$-th row of $\mathbf{AB}$ is:
\[
(\mathbf{AB})_i^{\top} = \sum_{k=1}^{n} a_{ik} \mathbf{b}_k^{\top}.
\]

Alternatively, we can use the fact that the row rank equals the column rank. By taking transposes:
\[
\operatorname{rank}(\mathbf{AB}) = \operatorname{rank}((\mathbf{AB})^{\top}) = \operatorname{rank}(\mathbf{B}^{\top}\mathbf{A}^{\top}).
\]

Applying Part 1 to $\mathbf{B}^{\top}\mathbf{A}^{\top}$:
\[
\operatorname{rank}(\mathbf{B}^{\top}\mathbf{A}^{\top}) \le \operatorname{rank}(\mathbf{A}^{\top}) = \operatorname{rank}(\mathbf{A}).
\]

Therefore:
\[
\operatorname{rank}(\mathbf{AB}) \le \operatorname{rank}(\mathbf{A}).
\]

Combining both parts, we have:

\begin{align}
\operatorname{rank}(\mathbf{AB}) \le 
\min\bigl\{\operatorname{rank}(\mathbf{A}),\, \operatorname{rank}(\mathbf{B})\bigr\}.
\label{eq:rank-product}
\end{align}

\end{proof}

\paragraph{Application to the Main Analysis in Submatrices}

Connection to Section \ref{sec:TheoreticalAnalysis}: The above lemmas directly support the rank analysis in the main paper:

\begin{enumerate}

\item \textit{Hadamard product structure:} In the unified formulation (Equation \ref{Unified-f}), Linear attention and Mamba mechanisms involve Hadamard products:
\begin{itemize}
\item Mamba: $\mathbf{M} = \mathbf{L_M} \circ (\mathbf{C}^{\top}\mathbf{B})$
\item Linear Attention: $\mathbf{M} = \mathbf{L_{Attn}} \circ (\psi(\mathbf{Q})\psi(\mathbf{K})^{\top})$
\end{itemize}

Lemma~\ref{lem:hadamard_rank} provides the upper bounds:
\begin{equation}
\label{upbound_lm}
\left\{
\begin{aligned}
R_{\text{LinAttn}}^{\text{off}} &\le \operatorname{rank}(\mathbf{L^{\text{off}}_{Attn}}) \cdot \operatorname{rank}(\psi(\mathbf{Q})\psi(\mathbf{K})^{\top}) \\
R_{\text{Mamba}}^{\text{off}} &\le \operatorname{rank}(\mathbf{L^{\text{off}}_M}) \cdot \operatorname{rank}(\mathbf{C}^{\top}\mathbf{B})
\end{aligned}
\right.
\end{equation}

\item \textit{Matrix product bounds}: For the terms $\psi(\mathbf{Q})\psi(\mathbf{K})^{\top}$ and $\mathbf{C}^{\top}\mathbf{B}$ appearing in the above equations, Lemma~\ref{lem:matrix_product_rank} provides:

\begin{equation}
\label{upbound_lm2}
\resizebox{\linewidth}{!}{$
\begin{aligned}
\operatorname{rank}(\psi(\mathbf{Q})\psi(\mathbf{K})^{\top}) 
  &\le \min\big\{\operatorname{rank}(\psi(\mathbf{Q})), 
  \operatorname{rank}(\psi(\mathbf{K})^{\top})\big\} 
  \le D_{QK}, \\[3pt]
\operatorname{rank}(\mathbf{C}^{\top}\mathbf{B}) 
  &\le \min\big\{\operatorname{rank}(\mathbf{C}^{\top}), 
  \operatorname{rank}(\mathbf{B})\big\} 
  \le N.
\end{aligned}
$}
\end{equation}

These bounds are crucial because they show that even before considering the masking operations, the representational capacity is already limited by the dimensionality of the intermediate transformations.

The key distinction between mechanisms arises from the structure of the off-diagonal submatrices. For a $C \times C$ off-diagonal block:

\textit{Linear Attention off-diagonal block}:
\[
\mathbf{L^{\text{off}}_{Attn}} = 
\begin{bmatrix}
1 & 1 & \cdots & 1 \\
1 & 1 & \cdots & 1 \\
\vdots & \vdots & \ddots & \vdots \\
1 & 1 & \cdots & 1
\end{bmatrix}_{C \times C}
\]
This is a fixed matrix of all ones, with $\operatorname{rank}(\mathbf{L^{\text{off}}_{Attn}}) = 1$.

\textit{Mamba off-diagonal block} (with block indices $i > j$):
\resizebox{\linewidth}{!}{$
\mathbf{L^{\text{off}}_M} = 
\begin{bmatrix}
A_{i,j+1} & A_{i,j+2} & \cdots & A_{i,j+C} \\
A_{i+1,j+1} & A_{i+1,j+2} & \cdots & A_{i+1,j+C} \\
\vdots & \vdots & \ddots & \vdots \\
A_{i+C-1,j+1} & A_{i+C-1,j+2} & \cdots & A_{i+C-1,j+C}
\end{bmatrix}_{C \times C}
$}

where $A_{p,q} = \mathbf{A}_p \mathbf{A}_{p-1} \cdots \mathbf{A}_{q+1}$ is a product of learnable parameters, with $\operatorname{rank}(\mathbf{L^{\text{off}}_M}) \ge 1$ and typically much higher.

Combining Equations~\ref{upbound_lm} and~\ref{upbound_lm2}, we obtain:
\begin{align*}
R_{\text{LinAttn}}^{\text{off}} &\le 1 \cdot D_{QK} = D_{QK} \\
R_{\text{Mamba}}^{\text{off}} &\le \operatorname{rank}(\mathbf{L^{\text{off}}_M}) \cdot N
\end{align*}

This establishes the rank hierarchy:
\[
\underbrace{C}_{\substack{\text{Self-Attn} \\ \text{(full rank)}}} > \underbrace{\operatorname{rank}(\mathbf{L^{\text{off}}_M}) \cdot N}_{\substack{\text{Mamba} \\ \text{(learnable mask)}}} > \underbrace{D_{QK}}_{\substack{\text{Linear Attn} \\ \text{(fixed mask)}}}.
\]

\item \textit{Block structure analysis:} The matrix $\mathbf{M}$ in Equation~\eqref{Unified-f} is partitioned into a $\frac{L}{C} \times \frac{L}{C}$ grid of $C \times C$ submatrices. We analyze the rank of diagonal and off-diagonal blocks separately for each mechanism.

\textit{Diagonal blocks:} All three mechanisms achieve full rank due to the lower triangular structure:

\resizebox{\linewidth}{!}{$
\mathbf{L}^{\text{diag}} = \begin{bmatrix}
1 & 0 & 0 & \cdots & 0 \\
\cdot & 1 & 0 & \cdots & 0 \\
\cdot & \cdot & 1 & \cdots & 0 \\
\vdots & \vdots & \vdots & \ddots & \vdots \\
\cdot & \cdot & \cdot & \cdots & 1
\end{bmatrix}_{C \times C}
\text{ with } R^{\text{diag}}_{\text{Self-Attn}} = R^{\text{diag}}_{\text{Mamba}} = R^{\text{diag}}_{\text{LinAttn}} = C
$}

\textit{Off-diagonal blocks:} The key differences emerge in the off-diagonal submatrices:

Self-Attention:
The softmax operation ensures full rank:
\[
R^{\text{off}}_{\text{Self-Attn}} = C.
\]

Mamba:
From Equations~\eqref{upbound_lm} and~\eqref{upbound_lm2}:
\[
R^{\text{off}}_{\text{Mamba}} \le \operatorname{rank}(\mathbf{L^{\text{off}}_M}) \cdot N.
\]

Linear Attention:
From Equations~\eqref{upbound_lm} and~\eqref{upbound_lm2}:
\[
R^{\text{off}}_{\text{LinAttn}} \le 1 \cdot D_{QK} = D_{QK}.
\]

For typical base model settings ($C = 256$, $N = D_{QK} = 64$):
\begin{center}
\resizebox{\linewidth}{!}{
\begin{tabular}{lccc}
\toprule
\textbf{Block Type} & \textbf{Self-Attn} & \textbf{Mamba} & \textbf{Linear Attn} \\
\midrule
Diagonal & $C = 256$ & $C = 256$ & $C = 256$ \\
Off-diagonal & $C = 256$ & $\le \operatorname{rank}(\mathbf{L_M}) \cdot 64$ & $\le 64$ \\
\bottomrule
\end{tabular}
}
\end{center}

This establishes the rank hierarchy for off-diagonal blocks:
\[
\underbrace{256}_{\text{Self-Attn}} > \underbrace{\text{rank}(\mathbf{L^{\text{off}}_M}) \cdot 64}_{\text{Mamba}} > \underbrace{64}_{\text{Linear Attn}}.
\]

Key observations:
\begin{itemize}
\item All mechanisms are equivalent on diagonal blocks ($R^{\text{diag}} = C$).
\item Self-Attention achieves full rank on off-diagonal blocks due to softmax nonlinearity.
\item Mamba's learnable mask $\mathbf{L_M}$ allows $\operatorname{rank}(\mathbf{L^{\text{off}}_M}) \gg 1$ in practice, significantly exceeding Linear Attention's fixed rank-1 mask.
\item Linear Attention is most constrained with $R^{\text{off}}_{\text{LinAttn}} \le D_{QK}$ due to its fixed all-ones mask.
\end{itemize}

\end{enumerate}

\paragraph{Remarks on Practical Implications}

The theoretical framework established here explains why:

\begin{enumerate}
\item \textit{Self-Attention is most expressive:} The softmax operation makes the effective rank of off-diagonal blocks full ($R^{\text{off}} = C$), allowing maximum representational capacity.

\item \textit{Mamba balances efficiency and expressiveness:} The learnable matrix $\mathbf{A}$ in $\mathbf{L_M^{\text{off}}}$ increases $\operatorname{rank}(\mathbf{L_M^{\text{off}}})$, enabling $R^{\text{off}}_{\text{Mamba}}$ to approach $N \cdot \operatorname{rank}(\mathbf{L_M^{\text{off}}})$. Since $N$ scales linearly with computation ($O(LNJ)$), Mamba can achieve higher ranks than Linear Attention while maintaining linear complexity. Moreover, Lemma~\ref{lem:matrix_product_rank} shows that the rank is fundamentally limited by $N$, but this limit is more attainable than in Linear Attention due to the learnable mask.

\item \textit{Linear Attention is most constrained:} The fixed mask $\mathbf{L_{Attn}^{\text{off}}}$ with $\operatorname{rank}(\mathbf{L_{Attn}^{\text{off}}}) = 1$ severely limits $R_{\text{LinAttn}}^{\text{off}} \le D_{QK}$ by Lemma~\ref{lem:hadamard_rank}. Moreover, Linear Attention's $O(LD_{QK}^2)$ complexity restricts $D_{QK}$ growth, further limiting its representational capacity. Lemma~\ref{lem:matrix_product_rank} additionally shows that the product $\psi(\mathbf{Q})\psi(\mathbf{K})^{\top}$ cannot exceed rank $D_{QK}$ regardless of the quality of the feature mappings.
\end{enumerate}

These mathematical foundations rigorously support the empirical observations in the experiments (Section \ref{sec:Experiments}), where feature map quality and downstream task performance follow the predicted hierarchy: ViT $>$ Mamba $>$ Linear ViT.

\end{document}